\documentclass[twoside,11pt]{article}
\pdfoutput=1
\usepackage{blindtext}
\usepackage[abbrvbib,preprint]{jmlr2e}

\usepackage{xcolor}
\definecolor{myMaroon}{RGB}{128, 0, 0}

\usepackage{algorithm}
\usepackage{algorithmic}

\usepackage{amsfonts}
\usepackage{amsmath}
\usepackage{amstext}
\usepackage{caption}
\usepackage{dsfont}
\usepackage{subcaption}
\usepackage{microtype}      
\usepackage{booktabs}
\usepackage{dirtytalk}
\usepackage{enumitem}
\usepackage{mathtools}
\usepackage{mathrsfs}
\usepackage{bbm}

\usepackage[capitalize]{cleveref}

\usepackage{nicefrac}

\usepackage{tabularx}
\newcolumntype{Y}{>{\centering\arraybackslash}X}
\newcolumntype{Z}{>{\raggedleft\arraybackslash}X}
\usepackage{tikz}
\usetikzlibrary{calc, fadings, decorations, shapes, positioning, arrows}

\DeclarePairedDelimiter\abs{\lvert}{\rvert}
\DeclarePairedDelimiter\norm{\lVert}{\rVert}

\DeclareMathOperator*{\argmin}{arg\,min}
\DeclareMathOperator*{\arginf}{arg\,inf}

\newcommand{\realNumber}{\mathbb{R}}
\newcommand{\naturalNumber}{\mathbb{N}}

\newcommand{\probP}{\mathds{P}}
\newcommand{\expectE}{\mathds{E}}

\newcommand{\indep}{\perp \!\!\! \perp}
\newcommand{\ill}{\nu}

\DeclarePairedDelimiterX{\expectarg}[1]{[}{]}{%
  \ifnum\currentgrouptype=16 \else\begingroup\fi
  \activatebar#1
  \ifnum\currentgrouptype=16 \else\endgroup\fi
}

\newcommand{\dataset}{\mathcal{D}}
\newcommand{\orthoM}{\psi}

\usepackage{mleftright}

\crefname{lemma}{Lemma}{Lemmas}
\crefalias{lemma}{theorem}
\crefname{assumption}{Assumption}{Assumptions}
\crefalias{assumption}{theorem}
\crefname{proposition}{Proposition}{Propositions}
\crefalias{proposition}{theorem}

\def\[#1\]{\begin{align*}#1\end{align*}}

\NewDocumentCommand{\numberthis}{om}{%
  \IfNoValueTF{#1}{%
    \refstepcounter{equation}\tag{\theequation}%
  }{%
    \tag{#1}%
  }%
  \label{#2}%
}

\AtEndPreamble{
\hypersetup{
    colorlinks=true,       
    linkcolor=blue,        
    citecolor=myMaroon,         
    filecolor=magenta,     
    urlcolor=cyan          
}
}

\newcounter{savetheorem}

\NewDocumentEnvironment{repeatthm}{m o}{%
  \setcounter{savetheorem}{\value{theorem}}%
  \edef\orignum{\getrefnumber{#1}}%
  \setcounter{theorem}{\numexpr\orignum-1\relax}%
  \begingroup
    \IfValueTF{#2}
      {\begin{theorem}[#2]}%
      {\begin{theorem}}%
}{%
    \end{theorem}%
  \endgroup
  \setcounter{theorem}{\value{savetheorem}}%
}

\NewDocumentEnvironment{repeatlemma}{m o}{%
  \setcounter{savetheorem}{\value{theorem}}%
  \edef\orignum{\getrefnumber{#1}}%
  \setcounter{theorem}{\numexpr\orignum-1\relax}%
  \begingroup
    \IfValueTF{#2}
      {\begin{lemma}[#2]}%
      {\begin{lemma}}%
}{%
    \end{lemma}%
  \endgroup
  \setcounter{theorem}{\value{savetheorem}}%
}

\NewDocumentEnvironment{repeatprop}{m o}{%
  \setcounter{savetheorem}{\value{theorem}}%
  \edef\orignum{\getrefnumber{#1}}%
  \setcounter{theorem}{\numexpr\orignum-1\relax}%
  \begingroup
    \IfValueTF{#2}
      {\begin{proposition}[#2]}%
      {\begin{proposition}}%
}{%
    \end{proposition}%
  \endgroup
  \setcounter{theorem}{\value{savetheorem}}%
}

\NewDocumentEnvironment{repeatcoro}{m o}{%
  \setcounter{savetheorem}{\value{theorem}}%
  \edef\orignum{\getrefnumber{#1}}%
  \setcounter{theorem}{\numexpr\orignum-1\relax}%
  \begingroup
    \IfValueTF{#2}
      {\begin{corollary}[#2]}%
      {\begin{corollary}}%
}{%
    \end{corollary}%
  \endgroup
  \setcounter{theorem}{\value{savetheorem}}%
}

\usepackage{lastpage}
\jmlrheading{25}{2025}{1-\pageref{LastPage}}{06/25; Revised **/**}{**/**}{21-0000}{Daqian Shao and Ashkan Soleymani and Francesco Quinzan and Marta Kwiatkowska}

\ShortHeadings{DML for Conditional Moment Restrictions}{Shao and Soleymani and Quinzan and Kwiatkowska}
\firstpageno{1}

\begin{document}

\title{Double Machine Learning for Conditional Moment Restrictions: IV Regression, Proximal Causal Learning and Beyond}

\author{\name Daqian Shao \email daqian.shao@cs.ox.ac.uk \\
       \addr Department of Computer Science\\
       University of Oxford\\
       Oxford, UK
       \AND
       \name Ashkan Soleymani \email ashkanso@mit.edu \\
       \addr Department of Electrical Engineering and Computer Science\\
       Massachusetts Institute of Technology\\
       Cambridge, USA
      \AND
       \name Francesco Quinzan\thanks{Current address and email: Department of Engineering Science, University of Oxford, UK. francesco.quinzan@eng.ox.ac.uk} 
       \email francesco.quinzan@cs.ox.ac.uk \\
       \addr Department of Computer Science\\
       University of Oxford\\
       Oxford, UK    
       \AND
       \name Marta Kwiatkowska \email
       marta.kwiatkowska@cs.ox.ac.uk\\
       \addr Department of Computer Science\\
       University of Oxford\\
       Oxford, UK
       }

\editor{undecided}

\maketitle

\begin{abstract}
Solving conditional moment restrictions (CMRs) is a key problem considered in statistics, causal inference, and econometrics, where the aim is to solve for a function of interest that satisfies some conditional moment equalities. Specifically, many techniques for causal inference, such as instrumental variable (IV) regression and proximal causal learning (PCL), are CMR problems. Most CMR estimators use a two-stage approach, where the first-stage estimation is directly plugged into the second stage to estimate the function of interest. However, naively plugging in the first-stage estimator can cause heavy bias in the second stage. This is particularly the case for recently proposed CMR estimators that use deep neural network (DNN) estimators for both stages, where regularisation and overfitting bias is present. We propose DML-CMR, a two-stage CMR estimator that provides an unbiased estimate with fast convergence rate guarantees. We derive a novel learning objective to reduce bias and develop the DML-CMR algorithm following the \textit{double/debiased machine learning (DML)} framework. We show that our DML-CMR estimator can achieve the minimax optimal convergence rate of $O(N^{-1/2})$ under parameterisation and mild regularity conditions, where $N$ is the sample size. We apply DML-CMR to a range of problems using DNN estimators, including IV regression and proximal causal learning on real-world datasets, demonstrating state-of-the-art performance against existing CMR estimators and algorithms tailored to those problems.
\end{abstract}


\section{Introduction}

In this work, we study the problem of \textit{conditional moment restrictions} (CMRs). 
Let $X\in\mathcal{X}\subseteq\realNumber^d$, $C\in\mathcal{C}\subseteq\realNumber^p$, and $Y\in\mathcal{Y}\subseteq\realNumber$ be random variables with their corresponding distributions. The CMR problem consists of estimating a function of interest $f_0\in\mathcal{F}\subseteq(\mathcal{X}\rightarrow\realNumber)$ in some hypothesis function space $\mathcal{F}$, such that 
\begin{align}
    \expectE[Y-f_0(X)\lvert C]=0.\tag{P} \label{eq:cmr}
\end{align}
This problem appears in many fields of research, such as statistical learning~\citep{Vapnik1998}, causal inference~\citep{Liao2020}, integral equations~\citep{Honerkamp1990}, deconvolution~\citep{Carrasco2007}, and econometrics~\citep{Carrasco2007}, and includes causal inference problems such as instrumental variable (IV) regression~\citep{wright1928} and proximal causal learning~\citep{Miao2018}.
%
%
%

Solving CMRs analytically is ill-posed~\citep{Nashed1974,kress1999linear} because it is an inverse problem that requires the derivation of the function $f_0$ inside the conditional expectation. Hence, various techniques have been proposed to estimate the solutions to CMR problems. The classic econometric approaches begin with the framework of generalised method of moments (GMM)~\citep{Hansen1982}, which gives rise to the classic two-stage least squares (2SLS)~\citep{Angrist1996} algorithm and sieve methods~\citep{Newey2003,ai2003efficient}. These techniques have strong theoretical foundations for convergence analysis, but they are restricted to the class of basis functions used in their theoretical analysis, and their empirical performance may suffer when estimating complex function classes and with high-dimensional input. More recent  works~\citep{Hartford2017DeepPrediction,Shao2024,Bennett2019DeepAnalysis,Singh2019}, inspired by the development of deep learning, proposed to use deep neural networks (DNNs) to parameterise and estimate the solutions to CMR problems. These methods allow for greater flexibility since they do not impose strong assumptions on the functional form such as linear functions or polynomials, and can learn directly from data with strong approximation power. 

Most of the existing CMR methods, including those that adopt DNNs, take a two-stage approach~\citep{Angrist1996,Newey2003,Chen2018,Singh2019,Muandet2020}. In the first stage, they estimate some nuisance parameters, which are parameters or infinite-dimensional functions of no direct interest, but are necessary for the second stage estimation. However, in these settings, regularisation is often employed to trade off overfitting with the induced regularisation bias, especially for high-dimensional inputs. This is problematic because both regularisation and overfitting can cause heavy bias~\citep{Chernozhukov2018Double/debiasedParameters} in two-stage estimations when the first stage estimator is naively plugged in for the second stage estimation, which results in slow convergence rate of the estimator.

In order to mitigate this problem, we take inspiration from \textit{double/debiased machine learning}~\citep{Chernozhukov2018Double/debiasedParameters} (DML), which is a technique that provides an unbiased estimator with strong convergence rate guarantees for general two-stage regressions. DML relies on having a Neyman orthogonal~\citep{Neyman1965} score function to deal with regularisation bias, and uses cross-fitting, that is, an efficient form of (randomised) data splitting, to address overfitting bias. However, the use of DML for CMR estimation involving neural networks has not been explored to the best of our knowledge.

In this work, we propose a novel CMR estimator, referred to as DML-CMR, with fast convergence rate guarantees based on the DML framework. We derive a novel Neyman orthogonal score for CMR problems and design a cross-fitting regime such that, under mild regularity conditions, it is guaranteed to converge at the rate of $N^{-1/2}$ with high probability, where $N$ is the sample size. For empirical evaluation, we apply DML-CMR with DNNs to solve two common CMR problems, IV regression and proximal causal learning. We evaluate its performance on multiple benchmarks, where superior results are demonstrated compared to state-of-the-art (SOTA) methods. Our main contributions are summarised below.

\begin{itemize}
\item We propose DML-CMR, 
a novel CMR estimator that leverages the DML framework to provide unbiased estimates of the solution to CMR problems. To the best of our knowledge, this is the first work that uses DML for CMR estimation with neural networks.
\item We derive a novel, Neyman orthogonal, score function for CMR problems in \cref{sec:neyman}, which does not rely on influence functions or the classical techniques in~\citet{Chernozhukov2018Double/debiasedParameters}, and design a cross-fitting regime for the DML-CMR estimator to mitigate the bias in \cref{sec:dmliv}.
\item In~\cref{sec:theory}, we show an asymptotic convergence rate for the DML-CMR estimator at the rate of $O(N^{-1/2})$, which is minimax optimal under parameterisation and mild regularity conditions.
\item 
On a range of IV regression and proximal causal learning problems, including two real-world datasets, we experimentally demonstrate that DML-CMR outperforms existing SOTA methods in~\cref{sec:exp}.
\end{itemize}

This paper extends our earlier approach of~\cite{Shao2024}, which developed an IV regression algorithm, called DML-IV, that uses the DML framework to provide fast convergence rate guarantees. Compared to ~\citet{Shao2024}, we consider the more general class of CMR estimators leveraging the DML~\citep{Chernozhukov2018Double/debiasedParameters} framework, derive new theoretical results and analysis, and conduct new experiments for the causal proximal learning problem.
\section{Related Works}
%
%

\subsection{Conditional Moment Restriction}\label{lit:cmr}

The classic framework of the generalised method of moments (GMM)~\citep{Hansen1982} was first proposed to translate conditional moment restrictions into unconditional moments, which can thus be estimated under the GMM framework. However, finding the set of unconditional moments that fully capture the conditional moments is challenging. Sometimes, for nonlinear CMRs and functions of interest $f$, an infinite set of unconditional moments may be required to represent the conditional moments, and a misspecification of the unconditional moments can bias the results~\citep{Domnguez2004}. To mitigate some of these issues, \citet{Angrist1996} considered linear functions of interest and proposed the classic two-stage least squares (2SLS) algorithm as a special case of GMMs. Following this, many efforts aim to extend 2SLS to solve CMRs, where the function of interest is nonlinear or nonparametric. One common approach is the sieve method, which uses nonlinear basis functions. Sieve minimum distance (SMD) estimator~\citep{Newey2003,ai2003efficient} performs regression in two stages using an increasing set of nonlinear
basis functions as the number of samples increases. Later, a penalised version of SMD~\citep{Chen2012} was proposed to generalise SMD by allowing non-smooth residuals and high-dimensional function spaces. These sieve-GMM methods and later works~\citep{Blundell2007,Chen2018,Singh2019,Muandet2020} that consider different dictionaries of basis functions enjoy strong consistency and efficiency guarantees, 
but their flexibility is limited by the set of basis functions, and they can be sensitive to the choice of such functions and regularisation hyperparameters. Note that these works also perform the estimation in two stages.

More recently, various  CMR estimation methods based on DNNs 
have been proposed, since machine learning methods can model highly nonlinear and high-dimensional relationships with greater flexibility. DeepIV~\citep{Hartford2017DeepPrediction} extended the classic 2SLS algorithm to the nonlinear setting by adopting DNNs for both stages with conditional density estimation~\citep{Darolles2011} in the first stage. GMM methods are also extended to use DNNs~\citep{Bennett2019DeepAnalysis,Liao2020,Dikkala2020,Bennett2020}, where a minimax criterion is optimised adversarially. Minimax approaches aim to solve a two-player zero-sum game where the players play adversarially. Specifically, player one chooses a hypothesis function to minimise moment violation, and player two chooses a test function that maximises moment violation. These DNN-based GMM methods require minimax optimisation, which is similar to the training of Generative Adversarial Networks~\citep{Goodfellow2014} and could be experimentally unstable.

In this work, we propose a doubly robust CMR estimator for nonlinear functions of interest $f$ that leverages the DML~\citep{Chernozhukov2018Double/debiasedParameters} framework and DNNs to provide fast convergence rate guarantees for the estimator, extending our previous IV regression algorithm, DML-IV, that also uses the DML framework~\citep{Shao2024}.

\subsubsection{IV regression}\label{lit:iv}
\textit{Instrumental variable} (IV) regression is a typical example of a CMR problem in causal inference, which aims to estimate the causal effect in the presence of hidden confounders (see~\Cref{appen:iv} for details). While most of the previously introduced CMR estimators can be directly applied to solve IV regression, here we review algorithms specifically designed for IV regression. Following the sieve-based approach, DFIV~\citep{Xu2020} proposed to use basis functions parameterised by DNNs, which remove restrictions on the functional form. In addition, Kernel IV~\citep{Singh2019} and Dual IV~\citep{Muandet2020} used different dictionaries of basis functions in reproducing kernel Hilbert spaces (RKHS) to solve the IV regression problem. DeepGMM~\citep{Bennett2019DeepAnalysis} is a DNN-based method that was inspired by GMM to solve IV regression using a minimax approach. \citet{Kremer2024} improved GMM-based IV regression methods in settings where the data manifold is not uniform through data-derivative information. With the exception of DML-IV~\citep{Shao2024}, the precursor of this work, none of these approaches utilise the DML framework.

\subsubsection{Proximal Causal Learning}\label{lit:proxy}
Another example of a CMR problem in causal inference is \textit{proximal causal learning} (PCL, see \Cref{appen:pcl} for details). PCL was first proposed by~\citet{Miao2018} to leverage two \textit{proxy variables} for causal identification in estimating the causal function. This was extended by~\citet{Shi2020} to a general semiparametric framework, where \citet{Tchetgen2020} introduced a two-stage procedure for linear causal models based on ordinary least squares regression. \citet{Mastouri2021} resolved how to handle nonlinear causal models by replacing linear regression with kernel ridge regression. To extend the kernel-based PCL methods, \citet{Xu2021} used DNNs as feature maps instead of fixed kernels. This improves the flexibility of the method, especially for highly nonlinear models. 
\citet{Kompa2022} proposed a single-stage PCL method based on 
maximum moment restrictions, where they train neural networks to minimise a loss function derived to satisfy the maximum moment restrictions. \citet{Cui2023} introduced a treatment bridge function and incorporated it into the Proximal Inverse Probability Weighting (PIPW) estimator. They considered only binary treatments and derived the Proximal Doubly Robust (PDR) estimator via influence functions. A similar approach by \citet{Wu2024} derived a doubly robust estimator for PCL with continuous treatment through influence functions, but none of these works adopted the DML framework, to the best of our knowledge.

The algorithms proposed specifically for IV regression and PCL often require additional problem-specific assumptions about variables and the functional form. We instead provide a general method for CMRs that can be directly applied to a range of problems, including IV regression and PCL.

\subsection{Double Machine Learning (DML)}\label{lit:dml}

DML was originally proposed for semiparametric regression~\citep{Robinson1988}. It relies on the derivation of a Neyman orthogonal~\citep{Neyman1965} score function that serves as the learning objective. DML was then extended by adopting DNNs for generalised linear regressions~\citep{Chernozhukov2021AutomaticRegression}. Its strength is that it provides unbiased estimates for two-stage estimations~\citep{Jung2021,Chernozhukov2022RieszNetForests} under certain identifiability conditions and offers $N^{-1/2}$ convergence rate guarantees.

There are previous works on combining DML with CMR estimation, specifically for the IV regression problem, but they are mainly focused on linear and partially linear functions of interest. \citet{belloni2012sparse} proposed a method to use Lasso and Post-Lasso methods for the first-stage estimation of linear IV to estimate the optimal instruments. To avoid selection biases, \citet{belloni2012sparse} leveraged techniques from weak identification robust inference. In addition, \citet{chernozhukov2015post} proposed a Neyman-orthogonalised score for the linear IV problem with control and instrument selection to potentially be robust to regularisation and selection biases of Lasso as a model selection method. Neyman orthogonality for partially linear models with IVs was mainly discussed in the work of \citet{Chernozhukov2018Double/debiasedParameters}. For an additional discussion, we refer to the book \citep{chernozhukov2024applied}.


DML for semiparametric models~\citep{Chernozhukov2022LocallyEstimation,Ichimura2022} has been previously applied to solve the nonparametric IV (NPIV) problem. However, their methods require additional assumptions on the IVs and residual functions such that the average moment of the Neyman orthogonal score is linear in the nuisance parameters. Such assumptions are not required in our work since we are considering a different problem setting and we formulate a novel Neyman orthogonal score. To the best of our knowledge, there is no work that adopts the DML framework for nonlinear IV regression and general CMR problems with DNNs.

\section{Preliminaries}

\section*{Notations}

We use uppercase letters such as $X$ to denote random variables and use the corresponding calligraphic letter such as $\mathcal{X}$ to denote the set from which the random variable takes its value. For example, $X\in\mathcal{X}\subseteq\realNumber^{d}$ is a $d$-dimensional real-valued random variable in $(\mathcal{X},\mathcal{B}_\mathcal{X})$, where $\mathcal{B}_\mathcal{X}$ is the Borel algebra on $\mathcal{X}$. 
An observed realisation of $X$ is denoted by a lowercase letter $x$. We abbreviate $\expectE[Y \lvert X=x]$, a realisation of the conditional expectation $\expectE[Y \lvert X]$, as $\expectE[Y \lvert x]$. $[N]$ denotes the set $\{1,...,N\}$ for $N\in\naturalNumber$. We use $\norm{\cdot}_p$ to denote the functional norm, defined as $\norm{f}_p\coloneqq\expectE[\abs{f(X)}^p]^{1/p}$, where the measure is implicit from the context. For a function $f$, we use $f_0$ to denote the true function and $\widehat{f}$ an estimator of the true function. We use $O$ and $o$ to denote big-O and little-o notations, respectively~\citep{Weisstein2023}.

\subsection{Conditional Moment Restrictions}\label{sec:prelim_CMR}

Recall that the CMR problem as in \Cref{eq:cmr} consists of providing an estimate for a function $f_0$ such that $\expectE[Y-f_0(X)\lvert C=c]=0$ for all $ c\in\mathcal{C}$, where $X\in\mathcal{X}\subseteq\realNumber^d$, $C\in\mathcal{C}\subseteq\realNumber^p$ and $Y\in\mathcal{Y}\subseteq\realNumber$ are random variables with their corresponding distributions.
%
As discussed in \Cref{lit:cmr}, many CMR estimators (e.g.,~\citealt{Angrist1996,Newey2003,Hartford2017DeepPrediction,Singh2019}) estimate $\widehat{f}$ in some space of functions $\mathcal{F}$ by solving the following objective function with a two-stage approach:
\begin{equation}
    \widehat{f} \in \arg\min_{f\in\mathcal{F}} \expectE[(Y-\expectE[f(X)\lvert C])^2]\label{eq:optim}.
\end{equation}
Specifically, \citet{Newey2003} take a minimax approach that indirectly optimises this objective by solving a minimax unconditional moment problem in two stages, as elaborated on in~\Cref{lit:cmr}. \citet{Angrist1996,Singh2019,Hartford2017DeepPrediction}, on the other hand, directly optimise the above objective in two stages. The first stage involves learning the conditional expectation $\expectE[f(X)\lvert c]$ using either density estimation or kernel methods from observations. In the second stage, the objective in~\Cref{eq:optim} is minimised using the estimations in the first stage. For both stages, linear regression, sieve methods, and DNNs are used for estimation, respectively, for each work.

\subsection{Double Machine Learning}\label{sec:dml}

DML considers the problem of estimating a function of interest $f$ as a solution to an equation of the form
\begin{equation}
\label{eq:score_function}
    \expectE[\orthoM(\dataset;f_0, \eta) ] =0,
\end{equation}
where $\orthoM$ is referred to as a score function and $f_0$ is the true function. Here, $\eta$ is a nuisance parameter, which can be of parametric form or an infinite-dimensional function. It is of no direct interest, but must be estimated to obtain an estimate of $f_0$. For example, in a two-stage CMR estimator, nuisance parameters such as conditional density are estimated in the first stage, and in the second stage, they are used to estimate $f_0$. DML provides a set of tools to derive an unbiased estimator of $f_0$ with convergence rate guarantees, even when the nuisance parameter $\eta$ suffers from regularisation, overfitting and other types of biases present in the training of ML models, which typically cause slow convergence when learning $f_0$.

In order to estimate $f_0$, DML reduces biases by using score functions $\orthoM$ that are Neyman orthogonal~\citep{Neyman1965} in $\eta$. This requires the Gateaux derivative, which defines the directional derivative for functionals, of the score function $\orthoM$ w.r.t. the nuisance parameters at $f_0,\eta_0$ to be zero:
\begin{align}
\label{eq:neyman}
\frac{\partial}{\partial r}\Big\lvert_{r=0} \expectE[\orthoM(\dataset;f_0,\eta_0+ r\eta)] = 0,
\end{align}
for all $\eta$. Here, $f_0$ and $\eta_0$ are the true parameters that minimise the expected score, that is, $\expectE[\orthoM(\dataset;f_0,\eta_0)] = 0$. Intuitively, the condition in~\Cref{eq:neyman} is met if small changes in the nuisance parameter do not significantly affect the score function around the true function $f_0$. Neyman orthogonality is key in DML, as it allows fast convergence for estimating $f_0$, even if the estimator for the nuisance parameter $\eta$ is biased. For score functions that are Neyman orthogonal, we define DML with \textit{K-fold cross-fitting} as follows.
\begin{definition}[DML, Definition 3.2~\citep{Chernozhukov2018Double/debiasedParameters}]\label{defn:dml}
Given a dataset $\mathcal{D}$ of $N$ observations, consider a score function $\orthoM$ as in~\Cref{eq:score_function}, and suppose that $\orthoM$ is Neyman orthogonal that satisfies~\Cref{eq:neyman}. Take a \textit{K-fold} random partition $\{I_k\}^K_{k=1}$ of observation indices $[N]$, each with size $n=N/K$, and let $\mathcal{D}_{I_k}$ be the set of observations $\{\mathcal{D}_i:i\in I_k\}$. Furthermore, define $I^c_k\coloneqq [N]\setminus I_k$ for each fold $k$, and construct estimators $\widehat{\eta}_k$ of the nuisance parameter using $\mathcal{D}_{I^c_k}$. Then, construct an estimator $\widehat{f}$ as a solution to the equation
\begin{align}
\label{eq:sol_model}
    \frac{1}{K}\sum_{k=1}^K \widehat{\expectE}_{k}[\orthoM(\dataset_{I_k};\widehat{f},\widehat{\eta}_k)]=0,
\end{align}
where $\widehat{\expectE}_k$ is the empirical expectation over $\mathcal{D}_{I_k}$.
\end{definition}
In the above definition, $\widehat{f}$ is defined as an exact solution to the empirical expectation equation in~\Cref{eq:sol_model}. In practice, we can also define the estimator $\widehat{f}$ as an approximate solution to~\Cref{eq:sol_model}\footnote{This approximation error is different to the estimation error. The estimation error measures the difference between $\widehat{f}$ and $f_0$, whereas the approximation error concerns the error of minimising the empirical risk. In fact, the approximation error contributes to the estimation error, which is analysed in~\Cref{sec:theory}.}.

\section{Solving CMRs with DML}\label{sec:3}

We now present the main contributions of this paper --- an estimator for solving CMRs under the DML framework. We propose a novel Neyman orthogonal score for our estimator and a novel two-stage algorithm for solving CMRs that can use DNN estimators in both stages and provides guarantees on the convergence rate by leveraging the DML framework.

To solve the CMR problem defined in~\Cref{eq:cmr} using the DML framework, we first need a Neyman orthogonal score. As introduced in~\Cref{sec:prelim_CMR}, the optimisation objective for many two-stage CMR estimators~\citep{Angrist1996,Hartford2017DeepPrediction,Singh2019} is~\Cref{eq:optim}. Let $g_0(f,c)\coloneqq\expectE[f(X)\lvert c]$ and $\mathcal{G}$ be some function space that includes $g_0$ and its potential estimators $\widehat{g}$. Then, these two-stage CMR estimators estimate $g_0$ with $\widehat{g}$ in the first stage, and then use $\widehat{g}$ to optimise the following loss function, $\ell=(Y-\widehat{g}(f,c))^2$, in the second stage. Unfortunately, as we show in Proposition~\ref{prop:standard_score} with proof deferred to~\Cref{appen:score}, this objective, or the score function, is not Neyman orthogonal. 

\begin{proposition}\label{prop:standard_score}
The score (or objective) function for standard two-stage CMR estimators, 
$\ell=(Y-\widehat{g}(f,c))^2$, 
is not Neyman orthogonal at $(f_0, g_0)$.
\end{proposition}

This means that small misspecifications or estimation biases of $\widehat{g}$ can lead to significant changes to the score function, and there are no guarantees on the convergence rate if the first stage estimator $\widehat{g}$ is naively plugged into the second stage to estimate $f_0$. To address this, we first derive a novel Neyman orthogonal score function for the CMR problem and then design a CMR algorithm with K-fold cross-fitting that uses the DML framework.

\subsection{Neyman Orthogonal Score}\label{sec:neyman}

Typically, to construct a Neyman orthogonal score from a non-orthogonal score, additional nuisance parameters need to be estimated~\citep{Chernozhukov2018Double/debiasedParameters}. These additional nuisance terms adjust the score in a way that makes it orthogonal, where the error in estimating $f_0$ due to errors in the nuisance parameters becomes second order in the Taylor expansion~\citep{Foster2019OrthogonalLearning}. In our case, to construct a Neyman orthogonal score for CMR problems from the standard objective in~\Cref{eq:optim}, we first select relevant functions that should be estimated as nuisance parameters. Following two-stage IV regression approaches~\citep{Hartford2017DeepPrediction}, estimating $g_0$ is essential for identifying $f_0$, so we will estimate it as a nuisance parameter. We find that, by additionally estimating $s_0(c)\coloneqq\expectE[Y\lvert c]$ inside some function space $\mathcal{S}$, we can construct the score function:
\begin{equation}
\orthoM(\dataset;f,(s,g))=(s(c)-g(f,c))^2.\label{eq:neyman_score}
\end{equation}
Here, the nuisance parameters are $\eta=(s,g)$. For this to be a valid Neyman orthogonal score function, we check that $\expectE[\orthoM(\dataset;f_0,(s_0,g_0))]=0$ with the true functions $(s_0,g_0)$, and its Gateaux derivative vanishes at $(f_0,(s_0,g_0))$ with the following theorem, where the proof is deferred to~\Cref{appen:neyman}.
\begin{theorem}[Neyman orthogonality]\label{thm:neyman}
The score function $\orthoM(\dataset;f,(s,g))=(s(c)-g(f,c))^2$ obeys the Neyman orthogonality conditions at $(f_0,(s_0,g_0))$.
\end{theorem}
This Neyman orthogonal score function is abstract in the sense that it allows for general estimation methods for $g_0$ and $s_0$, as long as they satisfy certain convergence conditions, which are introduced in the next section. In addition, having a Neyman orthogonal score is useful in general to debias two-stage estimators~\citep{Foster2019OrthogonalLearning}, beyond the specific DML regimes we are considering in this paper.

\subsection{A DML Estimator for Solving CMRs}\label{sec:dmliv}

With the Neyman orthogonal score, we now propose a novel DML estimator, DML-CMR, that solves~\Cref{eq:cmr}. Note that, in general, $f_0$ is allowed to be infinite-dimensional, as commonly seen in the nonparametric IV literature~\citep{Newey2003}. We also allow $f_0$ to be infinite-dimensional for the Neyman orthogonal score introduced in~\Cref{sec:neyman}. For the theoretical analysis of DML-CMR, while it is possible to provide a general analysis following~\citet{Foster2019OrthogonalLearning} for nonparametric $f_0$ with the Neyman orthogonal score, the analysis would require more assumptions and the convergence rate will depend on the complexity of the function classes involved in the estimation. For our analysis, since we propose a concrete estimator, we would like to provide a concrete analysis following the DML framework~\citep{Chernozhukov2018Double/debiasedParameters}, which is designed for semiparametric estimation, to show an optimal parametric rate for DML-CMR. Therefore, we assume that $f_0$ is finite-dimensional and parameterised for the theoretical analysis of DML-CMR.

\begin{assumption}[Parameterisation]\label{assump:parameter}
Let $f_0=f_{\theta_0}$ and $\Theta\subseteq \realNumber^{d_\theta}$ be a compact space of parameters of $f$, where the true parameter $\theta_0\in\Theta$ is in the interior of $\Theta$.
\end{assumption}
From this assumption, we can define $\mathcal{F}\coloneqq\{f_\theta:\theta\in\Theta\}$ as the function space of $f$.

\paragraph{The DML-CMR Estimator.} 
The procedure of our DML-CMR estimator with k-fold cross-fitting is outlined in~\Cref{alg:dml-iv-kf}. Given a dataset $\dataset=(y_i,x_i,c_i)_{i\in[N]}$ of size $N$, we first split the dataset using a random partition $\{I_k\}^K_{k=1}$ of dataset indices $[N]$ such that the size of each fold $I_k$ is $N/K$, 
and let $\mathcal{D}_{I_k}$ denote the set of observations $\{\mathcal{D}_i:i\in I_k\}$.

\begin{algorithm}[tb]
   \caption{DML-CMR with K-fold cross-fitting}
   \label{alg:dml-iv-kf}
\begin{algorithmic}[1]
   \STATE {\bfseries Input:} Dataset $\dataset$ of size $N$, number of folds $K$ for cross-fitting, mini-batch size $n_b$
   \STATE Get a partition $(I_k)^K_{k=1}$ of dataset indices $[N]$
   \FOR{$k=1$ {\bfseries to} $K$}
   \STATE $I^c_k\coloneqq[N]\setminus I_k$
   \STATE Learn $\widehat{s}_k$ and $\widehat{g}_k$ using $\{(\dataset_i):{i\in I^c_k}\}$
   \ENDFOR
   \STATE Initialise $f_{\widehat{\theta}}$
   \REPEAT
   \FOR{$k=1$ {\bfseries to} $K$}
   \STATE Sample $n_b$ data $c_i^k$ from $\{(\dataset_i):{i\in I_k}\}$
   \STATE $\mathcal{L}=\widehat{\expectE}_{c_i^k}\left[(\widehat{s}_k(c)-\widehat{g}_k(f_{\widehat{\theta}},c))^2\right]$
   \STATE Update $\widehat{\theta}$ to minimise loss $\mathcal{L}$ 
   \ENDFOR
    \UNTIL{convergence}
    \STATE {\bfseries Output:} The DML-CMR estimator $f_{\widehat{\theta}}$
\end{algorithmic}
\end{algorithm}

As introduced in~\Cref{sec:dml}, our DML estimator will be a two-stage procedure. In the first stage (lines 4-7 in~\Cref{alg:dml-iv-kf}), for each fold $k\in [K]$, we learn $\widehat{s}_k$ and $\widehat{g}_k$ using data $\dataset_{I^c_k}$ with indices $I^c_k\coloneqq[N]\setminus I_k$. Then $\widehat{s}_k\approx\expectE[Y\lvert C]$ can be learnt through standard supervised learning using a neural network with inputs $C$ and labels $Y$. For $\widehat{g}_k$, we follow~\citet{Hartford2017DeepPrediction} to estimate $F_0(X\lvert C)$, the conditional distribution of $X$ given $C$, with $\widehat{F}$, and then estimate $\widehat{g}$ via
\begin{equation*}
\widehat{g}(f_\theta,c)=\sum_{\dot{X}\sim \widehat{F}(X\lvert C)} f_\theta(\dot{X})\approx \int f_\theta(X)\widehat{F}(X\lvert C=c) dX\approx\expectE[f_\theta(X)\lvert c].
\end{equation*}
For example, if the action space is discrete, $\widehat{F}$ can be a categorical model, e.g., a DNN with softmax output. For a continuous action space, a mixture of Gaussian models can be adopted to estimate the distribution $F_0(X\lvert C)$, where a DNN is used to predict the mean and standard deviation of the Gaussian distributions.

In the second stage (lines 8-15 in~\Cref{alg:dml-iv-kf}), we estimate $\widehat{\theta}$ using our Neyman orthogonal score function $\orthoM$ in~\Cref{eq:neyman_score}. The key is to optimise $\widehat{\theta}$ with data from the $k$-th fold $\mathcal{D}_{I_k}$ using nuisance parameters $\widehat{s}_k$, $\widehat{g}_k$ that are trained with $\dataset_{I^c_k}$, the complement of $\mathcal{D}_{I_k}$. This is important to fully debias the estimator $\widehat{\theta}$. The DML estimator $f_{\widehat{\theta}}$ is then defined as
\begin{align}
    f_{\widehat{\theta}}\coloneqq \min_{f_\theta\in \mathcal{F}}\frac{1}{K}\sum_{k=1}^K \widehat{\expectE}_{k}[(\widehat{s}_k(c)-\widehat{g}_k(f_{\widehat{\theta}},c)^2],\label{eq:dml_estimator}
\end{align}
where $\widehat{\expectE}_k$ is the empirical expectation over $\mathcal{D}_{I_k}$. In practice, we can alternate between the $K$ folds while sampling a mini-batch $c_i^k$ of size $n_b$ from each fold $\mathcal{D}_{I_k}$ to update $\widehat{\theta}$ by minimising the empirical loss on the mini-batch following our Neyman orthogonal score $\orthoM$,
\begin{equation*}
\mathcal{L}=\widehat{\expectE}_{c_i^k} \left[(\widehat{s}_k(c)-\widehat{g}_k(f_{\widehat{\theta}},c))^2\right] = \sum_{c_i^k}\frac{1}{n_b}\left((\widehat{s}_k(c)-\widehat{g}_k(f_{\widehat{\theta}},c))^2\right ).
\end{equation*}
When the second stage converges, we return the DML-CMR estimator $f_{\widehat{\theta}}$.

\subsection{Theoretical Analysis}\label{sec:theory}

In this section, we provide a theoretical analysis of the convergence of DML-CMR. The key benefit of DML is its debiasing effect for two-stage regressions, and crucially, it is possible to leverage the DML framework~\citep{Chernozhukov2018Double/debiasedParameters} to show a fast asymptotic convergence rate of $O(N^{-1/2})$, i.e., the DML estimator $\widehat{\theta}$ converges to the true parameters $\theta_0$ at the rate of $O(N^{-1/2})$ with high probability. To provide a road map of this section, we first list all the technical conditions required for a general DML estimator to converge at the fast rate of $O(N^{-1/2})$ following Theorem 3.3 from~\citep{Chernozhukov2018Double/debiasedParameters}. Later, in~\Cref{thm:dml}, we will prove that all these conditions hold for our DML-CMR estimator.

\begin{condition}[Technical conditions of DML $N^{-1/2}$ rate proved later in~\Cref{thm:dml}]\label{condition:dml}
For sample size $N\geq3$: 
\begin{enumerate}[label=(\alph*)]
\item The map $(\theta,(s,g))\mapsto \expectE[\orthoM(\dataset;f_{\theta},(s,g))]$ is twice continuously Gateaux-differentiable.
\item  The score $\orthoM$ obeys the Neyman orthogonality conditions in \Cref{eq:neyman}.
\item The true parameter $\theta_0$ obeys $\expectE[\orthoM(\dataset;f_{\theta_0},(s_0,g_0))]=0$ and $\Theta$ contains a ball of radius $c_1 N^{-1/2}\log N$ centered at $\theta_0$.
\item For all $\theta\in\Theta$, the identification relationship
\begin{align*}
2\norm{\expectE[\orthoM(\dataset;f_{\theta},(s_0,g_0))]}\gtrsim \norm{J_0(\theta-\theta_0)}
\end{align*}
is satisfied, where $J_0\coloneqq\partial_{\theta^\prime}\{\expectE[\orthoM(\dataset;f_{\theta^\prime},(s_0,g_0))]\}|_{\theta^\prime=\theta_0}$ is the Jacobian matrix, with singular values bounded between $c_0>0$ and $c_1>0$.
\item  All eigenvalues of the matrix $\expectE[\orthoM(\dataset;f_{\theta_0},(s_0,g_0))\orthoM(\dataset;f_{\theta_0},(s_0,g_0))^T]$ are strictly positive (bounded away from zero).
\item Let $K$ be a fixed integer. Given a random partition $\{I_k\}_{k=1}^K$ of indices $[N]$, each of size $n=N/K$, the nuisance parameter estimator $\widehat{s}_k$ and $\widehat{g}_k$ learnt using data with indices $I^c_k$ belongs to shrinking realisation sets $\mathcal{S}_N$ and $\mathcal{G}_N$, respectively, and the nuisance parameters should be estimated at the $o(N^{-1/4})$ rate, e.g., $\norm{\widehat{s}-s_{0}}_2=o(N^{-1/4})$.
\end{enumerate}
\end{condition}

Among these conditions, (a), (b), and (c) are conditions regarding the Neyman orthogonal score $\orthoM$. The Neyman orthogonality in (b) is shown in~\Cref{thm:neyman} and the other conditions (b) and (c) are mild regularity conditions, standard for moment problems. (d) is an identification condition that ensures sufficient identifiability of $\theta_0$. This condition also implies bounded ill-posedness, which we will discuss in detail in~\Cref{sec:ill-posedness}. (e) is a non-degeneracy assumption for the covariance of the score function. Finally, (f) is a key condition that states the nuisance parameters should converge to their true values at the crude rate of $o(N^{-1/4})$, where a shrinking realisation set $\mathcal{S}_N$ is a decreasing set of possible estimators $\widehat{s}$ as the sample size $N$ increases.

In Lemma~\ref{lemma:nuisances}, following recent works~\citep{Chernozhukov2018Double/debiasedParameters,Chernozhukov2021AutomaticRegression,Chernozhukov2022RieszNetForests}, we show that the convergence condition for the nuisance parameters in Condition~\ref{condition:dml} (f) can be transformed to a condition on the critical radius~\citep{bartlett2005local} of the realisation sets, which is a property widely studied for various function classes. For this analysis, we start by assuming the realisability of the true functions $g_0, s_0$ and $f_0$ in their corresponding function classes, and further assuming that they are bounded. We formalise these assumptions in Assumption~\ref{assump:dml}.

\begin{assumption}
[Realisable and bounded]\label{assump:dml}
We assume that $g_0, s_0, f_0$ are realisable in the function classes $\mathcal{G},\mathcal{S},\mathcal{F}$, that is, $g_0, s_0, f_0\in\mathcal{G},\mathcal{S},\mathcal{F}$, respectively, and furthermore, $\norm{f}_\infty, \norm{s}_\infty  \leq B$ for all $f, s \in \mathcal{F}, \mathcal{S}$, where $B$ is a positive constant. Moreover, we assume that the random variable $|Y| \leq B$ almost surely.
\end{assumption}

 As discussed in \Cref{sec:neyman}, DML-CMR has two nuisance parameters that are required to be estimated: $\widehat{s} \in \mathcal{S}$ and $\widehat{g} \in \mathcal{G}$. As we saw in \Cref{sec:dmliv}, the estimation of $\widehat{s}$ is made through standard supervised learning algorithms that we can directly analyse. However, the estimation of $\widehat{g}$ has two steps: 
 (i) we estimate the conditional distribution $\widehat{F}(X|C) \in \mathcal{P}$, where the density sieve $\mathcal{P}$ is defined as
\[
\mathcal{P} \subset \mleft\{ F : \int F(x | C = c) dx = 1 \quad \forall c \in \mathcal{C}\mright\};
\]
and (ii) we plug in the functional $f_\theta$ into the conditional expectation estimator,
\[
\widehat{g}_{\widehat{F}}(f_\theta, c) \coloneqq \int f_\theta(x) \widehat{F}(x | C = c) dx,
\]
for all candidate test functions $f_\theta \in \mathcal{F} = \mleft\{ f_\theta \;:\; \theta \in \Theta \mright\}.$
From the realisability assumption of $g_0 \in \mathcal{G}$ in Assumption~\ref{assump:dml}, it follows that $F_0(X|C) \in \mathcal{P}$, and the hypothesis space for the estimand $\widehat{g}$ and the true parameter $g_0$ is defined as $\mathcal{G} \coloneqq \mleft\{ g_{F} \;:\; F \in \mathcal{P}
 \mright\}$.

\begin{lemma}[Convergence of nuisance parameters]\label{lemma:nuisances}
     Under Assumption~\ref{assump:dml}, let $\mathcal{S}^*_N$ be the star-hull of the realisation set $\mathcal{S}_N$ of function class $\mathcal{S}$,
    \[
    \mathcal{S}^*_N = \mleft\{ C \mapsto \gamma(s(C) - s_0(C)) \;:\; s \in \mathcal{S}_N, \gamma \in [0, 1] \mright\},
    \]
    $\mathcal{P}^*_N$ be the star-hull of the realisation set $\mathcal{P}_N$ of the function class $\mathcal{P}$,
    \[
    \mathcal{P}^*_N = \mleft\{ C \mapsto \gamma(F(\cdot|C) - F_0(\cdot|C)) \;:\; F \in \mathcal{P}_N, \gamma \in [0, 1] \mright\},
    \]
    and $\mathcal{G}^*_N$ be the star-hull of the realisation set $\mathcal{G}_N$ of the function class $\mathcal{G}$,
    \[
    \mathcal{G}^*_N = \mleft\{C, f \mapsto\gamma(g(C,f)-g_0(C,f))\;:\; g\in\mathcal{G}_N,  \gamma\in[0,1]\mright\},
    \]
     where $\mathcal{S}_N$, $\mathcal{P}_N$ and $\mathcal{G}_N$ are properly shrinking neighbourhoods of the true functions $s_0$, $F_0$ and  $g_0$. Then, there exist universal constants $c_1$ and $c_2$, for which we have that, with probability at least $1 - \xi$, the estimation errors are bounded as
    \[
    \norm{\widehat{s}-s_0}_2^2 & \leq c_1 \mleft( \delta_N(\mathcal{S}_N^*)^2+\sqrt{\frac{\log(1/\zeta)}{N} }
    \mright) ; \\
     \norm{\widehat{g} - g_0}_2^2 & \leq c_2 \mleft( \delta_N(\mathcal{P}^*_N)^2+\sqrt{\frac{\log(1/\zeta)}{N}} \mright).
    \]
\end{lemma}

This lemma shows that, if we can upper bound the critical radii of $\delta_N(\mathcal{S^*})$ and $\delta_N(\mathcal{P^*})$ by $o(N^{-1/4})$, then $\norm{\widehat{s}-s_0}_2= o(N^{-1/4})$, and $\norm{\widehat{g}-g_0}_2 = o(N^{-1/4})$, meaning that nuisance parameters converge to their true values at the rate of $o(N^{-1/4})$ as required by Condition~\ref{condition:dml} (f). Next, we provide an analysis and concrete examples of estimators that satisfy this requirement.

The critical radius is a quantity that describes the complexity of estimation, and it is typically shown that $\delta_N=O(d_N^{1/2} N^{-1/2})$~\citep{Chernozhukov2022RieszNetForests,Chernozhukov2021AutomaticRegression} (see~\Cref{appen:critical_radius}), where $d_N$ is the effective dimension of the hypothesis space (for a formal definition of the critical radius and the effective dimension, the relationship of these metrics to Dudley's entropy integral, and bounds on the excess risk of estimators, refer to~\Cref{appen:critical_radius}.). This, together with Lemma~\ref{lemma:nuisances}, implies that $\norm{\widehat{s}-s_0}_2=O(d_N(\mathcal{S}^*)^{1/2} N^{-1/2})$. Therefore, we can also see that, if the effective dimension satisfies $d_N(\mathcal{S}^*)=o(N^{1/4})$, then $\norm{\widehat{s}-s_0}_2 = o(N^{-1/4})$ as required by Condition~\ref{condition:dml} (f) (and similarly for $\widehat{g}$ and $d_N(\mathcal{P}^*)$).

Therefore, we can refer to results in the literature that analyse the effective dimension and critical radius of various estimators to provide examples of estimators that satisfy Condition~\ref{condition:dml} (f). For the estimation of $\widehat{s}$, we have a regression problem and Condition~\ref{condition:dml} (f) is satisfied by many supervised learning estimators such as parametric generalised linear models~\citep{van1996weak}, Lasso~\citep{bickel2009simultaneous}, random forests~\citep{syrgkanis2020estimation}, boosting~\citep{luo2016high}, Sobolev kernel regression with $\alpha$-smooth RKHS ($\alpha > d /2$, where $d$ is the dimension of $X$)~\citep{caponnetto2007optimal,christmann2008support} and neural networks~\citep{chen1999improved,yarotsky2018optimal,schmidt2020nonparametric,farrell2021deep}. For the conditional density estimator $\widehat{g}$, the above estimators also satisfy Condition~\ref{condition:dml} (f) if the conditional distribution can be parameterised accordingly; otherwise, the condition is also satisfied by Gaussian mixtures (mixture density networks)~\citep{ho2022convergence}, polynomial sieve with Hölder smoothness $\alpha > \frac{d + 1}{2}$~\citep{ai2003efficient}, and categorical-logistic models~\citep{zhao2022asymptotic}, among others.

Lemma~\ref{lemma:nuisances} allows us to obtain the following theorem regarding the convergence of the DML-CMR estimator by applying Theorem 3.3 of~\citet{Chernozhukov2018Double/debiasedParameters}. We  prove satisfaction of all technical conditions for DML convergence rate guarantee mentioned in Condition~\ref{condition:dml} in  \Cref{appen:dml}.

\begin{theorem}[Convergence of the DML estimator for CMRs]\label{thm:dml}
Let $f_{\theta_0}\in\mathcal{F}$ be a solution that satisfies the CMRs in~\Cref{eq:cmr}, let $\orthoM$ be the Neyman orthogonal score defined in~\Cref{eq:neyman_score} and let $J_0\coloneqq\partial_{\theta^\prime}\{\expectE[\orthoM(\dataset;f_{\theta^\prime},(s_0,g_0))]\}|_{\theta^\prime=\theta_0}$ be the Jacobian matrix of $\expectE[\orthoM]$ w.r.t. $\theta$. Suppose that the upper bound of the critical radius $\delta_N=o(N^{-1/4})$, for $\widehat{s}$ and $\widehat{g}$, and $J_0$ have bounded singular values. Then, if Assumption~\ref{assump:parameter} and \ref{assump:dml} hold, our DML-CMR estimator $f_{\widehat{\theta}}$ satisfies that $\widehat{\theta}$ is concentrated in a $N^{-1/2}$ neighbourhood of $\theta_0$, and is approximately linear and centered Gaussian:
\begin{align*}
    \sqrt{N}(\widehat{\theta}-\theta_0)\rightarrow \mathcal{N}(0,\sigma^2) \text{ in distribution},
\end{align*}
where the estimator variance is given by
\begin{equation*}
\sigma^2 \coloneqq J_0^{-1}\expectE[\orthoM(\dataset,\theta_0,(s_0,g_0))\orthoM(\dataset,\theta_0,(s_0,g_0))^T](J_0^{-1})^T,
\end{equation*}
which is constant w.r.t. $N$.
\end{theorem}
\Cref{thm:dml} states that, with adequately trained nuisance parameter estimators in terms of their critical radius and identifiability conditions in terms of the non-singularity of the Jacobian matrix $J_0$, the estimator error $\widehat{\theta}-\theta_0$ is normally distributed, where its variance shrinks at the rate of $N^{-1/2}$. This implies that $\widehat{\theta}$ converges to $\theta_0$ at the rate $O(N^{-1/2})$ with high probability, which allows us to bound the estimation error $\norm{f_{\widehat{\theta}}-f_{\theta_0}}_2$ of the DML-CMR estimator with high probability, under a Lipschitz condition of $f_\theta$.

\begin{corollary}\label{coro:function_convergence}
Let $f_{\widehat{\theta}}$ be our DML-CMR estimator. If all assumptions for~\Cref{thm:dml} hold and there exists a constant $L>0$ such that $\norm{f_\theta(x)-f_{\theta_0}(x)}_2\leq L \norm{\theta-\theta_0}_2$ for all $x\in\mathcal{X}$ and $\theta\in\Theta$,
then for all $\zeta\in(0,1]$, we have that
\begin{align*}
\norm{f_{\widehat{\theta}}-f_{\theta_0}}_2=O\left(L\sqrt{\frac{\ln(1/\zeta)}{N}}\right),
\end{align*}
with probability $1-\zeta$.
\end{corollary}
    
Here, we assume a local Lipschitz condition of $f_\theta$ w.r.t. $\theta$ around $\theta_0$: $\norm{f_\theta(x)-f_{\theta_0}(x)}_2\leq L\norm{\theta-\theta_0}_2$ for all $x\in\mathcal{X}$ and $\theta\in\Theta$ for some $L>0$. Since $\Theta$ is compact, we see that it is enough for this Lipschitz condition to hold locally in a neighbourhood of $\theta_0$.

\subsection{DML Identifiability Condition and Ill-posedness}\label{sec:ill-posedness}

In this section, we provide a discussion regarding the relationship between our identifiability condition, Condition~\ref{condition:dml} (d), and a more common notion of identifiability for CMR problems, the \textit{ill-posedness}~\citep{Chen2012}. To begin with, we define the ill-posedness of CMR problems.

\begin{definition}[Ill-posedness~\citep{Chen2012,Dikkala2020}]\label{def:ill-posed}
Given a CMR problem as in~\Cref{eq:cmr}, the \emph{ill-posedness} $\ill$ of the function space $\mathcal{F}$ is given by
\begin{align*}
    \ill=\sup_{f\in\mathcal{F}} \frac{\norm{f_0-f}_{2}}{\norm{\expectE[f_0(X)-f(X) \lvert C]}_{2}}.
\end{align*}
\end{definition}

 Intuitively, ill-posedness describes how well a small CMR error (the projected error under conditional expectation) implies a small $L_2$ error (root mean squared error) to $f_0$. For identifiability, it is usually assumed that the ill-posedness $\ill$ is bounded. Otherwise, even solving the CMRs with very small error does not guarantee a solution $\widehat{f}$ that is close to $f_0$. In our case, we demonstrate that the identification condition of a DML estimator actually implies bounded ill-posedness. Specifically, Condition~\ref{condition:dml} (d) 
implies that the ill-posedness is bounded, as shown by the following proposition.

\begin{proposition}\label{prop:ill-posed}
For all $\theta\in\Theta$, if there exists a constant $L>0$ such that $\norm{f_\theta(x)-f_{\theta_0}(x)}_2\leq L \norm{\theta-\theta_0}_2$ for all $x\in\mathcal{X}$ and $\theta\in\Theta$, then Condition~\ref{condition:dml} (d), which states 
\begin{align*}
2\norm{\expectE[\orthoM(\dataset;f_{\theta},(s_0,g_0))]}\geq \norm{J_0(\theta-\theta_0)}
\end{align*}
and the Jacobian matrix $J_0$ have singular values bounded between $c_0>0$ and $c_1>0$, implies the ill-posedness is bounded by $\ill\leq L/\sqrt{c_0}$.
\end{proposition}

The proof of Proposition~\ref{prop:ill-posed} is deferred to~\Cref{appen:illposed}. This interesting finding explains why, in~\Cref{thm:dml}, there are no explicit assumptions about the ill-posedness of the problem. The only identifiability assumption for~\Cref{thm:dml} is the non-singularity of the Jacobian matrix $J_0$, which for DML-CMR is sufficient to ensure Condition~\ref{condition:dml} (d) holds. Therefore, by Proposition~\ref{prop:ill-posed}, this ensures a bounded ill-posedness of the problem, which in turn allows us to identify the solution $f_0$ with small error.
\section{Experimental Results}\label{sec:exp}

In this section, we empirically evaluate our DML-CMR estimator. We apply DML-CMR to two applications, IV regression and proximal causal learning, where details regarding these two problems are provided in~\Cref{appen:apply_cmrs}. In addition, we evaluate a computationally efficient version of DML-CMR, referred to as CE-DML-CMR, which does not apply $K$-fold cross-fitting. It trains $\widehat{s}$ and $\widehat{g}$ only once (instead of $K$ times) using the entire dataset, and can also be considered as an ablation study on $K$-fold cross-fitting. Without $K$-fold cross-fitting, it lacks the theoretical convergence rate guarantees but it still enjoys the partial debiasing effect~\citep{Mackey2018} from the Neyman orthogonal score and trades off computational complexity with bias. We found that CE-DML-CMR empirically performs as well as standard DML-CMR on low-dimensional datasets. We provide details and discussion regarding CE-DML-CMR in~\Cref{appen:cedml}.

Our evaluation considers both low- and high-dimensional datasets, as well as semi-synthetic real-world datasets. We ran each method 20 times and report the mean squared errors (MSE) between the estimators $\widehat{f}$ and $f_0$, where the median, 25th, and 75th percentiles are shown. 
The method employed by DML-CMR is identical to DML-IV~\citep{Shao2024} when solving the IV regression problem, and we include the results for IV regression from \citet{Shao2024} in this section. For PCL, the experimental results are new, and we implemented all algorithms using PyTorch~\citep{Paszke2019}. The full code is available on GitHub\footnote{\url{https://github.com/shaodaqian/DML-CMR}}. 

\subsection{IV regression}

For the IV regression task (details in~\Cref{appen:iv}), we compare our methods with leading modern IV regression methods Deep IV~\citep{Hartford2017DeepPrediction}, DeepGMM~\citep{Bennett2019DeepAnalysis}, KIV~\citep{Singh2019} and DFIV~\citep{Xu2020}.

We use DNN estimators for both stages with network architecture and hyperparameters provided in~\Cref{appen:networks}. Results of DML-CMR using tree-based estimators such as Random Forests and Gradient Boosting are provided in~\Cref{appen:tree-based}, where comparable performance to DNN-based DML-CMR is demonstrated. In addition, we provide a sensitivity analysis against hyperparameter changes in~\Cref{appen:sensitivity} and an evaluation of algorithms when the IV is weakly correlated with the treatment, representing  higher ill-posedness of the CMRs, in~\Cref{appen:weak_iv}.

\subsubsection{Ticket Demand Dataset}

\begin{figure*}[t]
\centering
\begin{subfigure}[c]
{0.6\textwidth}
\centering
\includegraphics[width=0.9\textwidth]{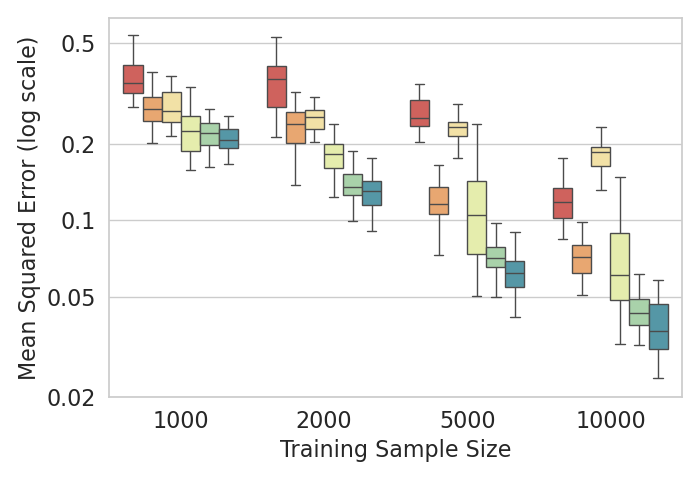}
\end{subfigure}
\begin{subfigure}[c]
{0.3\textwidth}
\centering
\includegraphics[width=0.6\textwidth]{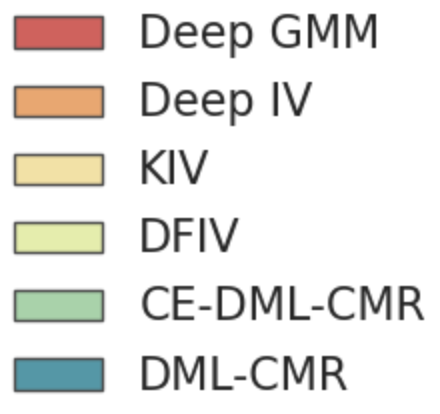}
\end{subfigure}
\caption{The mean squared error of $\widehat{f}$ on the ticket demand dataset with low-dimensional context for the IV regression task.}
\label{fig:lowd_mse}
\end{figure*}

We first conduct experiments for IV regression on the ticket demand dataset, which is a synthetic dataset introduced by~\citet{Hartford2017DeepPrediction} that is now a standard benchmark for nonlinear IV methods. In this dataset, we aim to understand how ticket prices $p$ affect ticket sales $r$. We observe two context variables, which are the time of year $t\in[0,10]$ and customer type $s\in[7]$ variables, the latter categorised by the level of price sensitivity. Price and context affect sales through $f_0((t,s),p)=100+(10+p)\cdot s \cdot \psi(t)-2p$, where $\psi(t)$ is a complex nonlinear function. However, the noise of $r$ and $p$ is correlated, which indicates the existence of unobserved confounders. The fuel price $z$ is introduced as an instrumental variable. Details of this dataset are included in~\Cref{appen:demand}.

The results for learning $f_0$ with this dataset of various sizes are provided in~\Cref{fig:lowd_mse}. It can be seen that DML-CMR performs better than other IV regression methods for all dataset sizes. CE-DML-CMR, which requires significantly less computation, matches the performance of DML-CMR in this case.

\subsubsection{High-Dimensional Dataset} 

\begin{figure*}[t]
\centering
\begin{subfigure}[c]
{0.6\textwidth}
\centering
\includegraphics[width=0.9\textwidth]{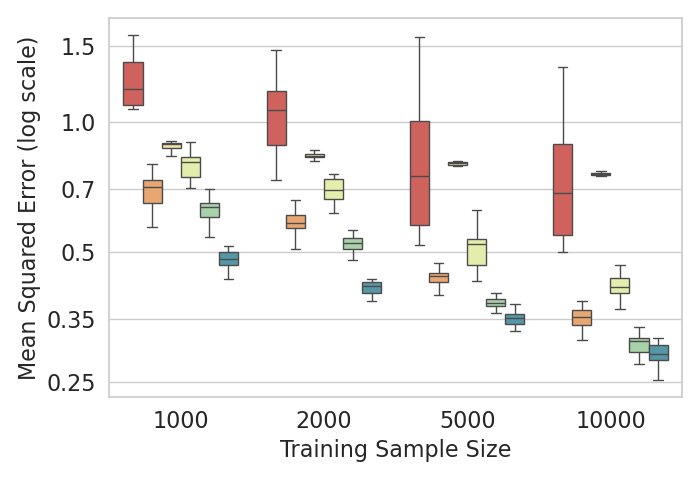}
\end{subfigure}
\begin{subfigure}[c]
{0.3\textwidth}
\centering
\includegraphics[width=0.6\textwidth]{figures/legend_iv.png}
\end{subfigure}
\caption{The mean squared error of $\widehat{f}$ on the ticket demand dataset with high-dimensional context for the IV regression task.}
\label{fig:mnist_mse}
\end{figure*}

In real applications, we
typically do not observe variables such as the customer type as explicit categories. Therefore, we follow~\citet{Hartford2017DeepPrediction} and consider the case where the customer type $s\in[7]$ is replaced by images of the corresponding handwritten digits from the MNIST dataset~\citep{LeCun2010} to evaluate our methods with high-dimensional ($28^2$=784 dimensions) inputs. The task remains to learn $f_0$, but the algorithms are no longer explicitly given the 7 customer types, and instead have to infer the relationship between the image data and the outcome. Results for IV regression are plotted in~\Cref{fig:mnist_mse}, where DML-CMR and CE-DML-CMR outperform all other methods. In these high-dimensional settings, regularisation is heavily used to avoid overfitting. DML-CMR demonstrates the benefits of using DML to reduce both the regularisation and overfitting bias caused by learning the nuisance parameters.

\subsubsection{Real-World Datsets}

Lastly, we test the performance of DML-CMR on real-world datasets. The true counterfactual prediction function is rarely available for real-world data. 
Therefore, in line with previous approaches~\citep{Shalit2017,Wu2023,Schwab2019,Bica2020}, we instead consider two semi-synthetic real-world datasets IHDP\footnote{IHDP: \url{https://www.fredjo.com/}.}~\citep{Hill2011} and PM-CMR\footnote{PM-CMR: \url{https://doi.org/10.23719/1506014}.}~\citep{Wyatt2020}. We directly use the continuous variables from IHDP and PM-CMR as context variables, and generate the outcome variable with a nonlinear synthetic function following~\citet{Wu2023}. There are 470 and 1350 training samples in IHDP and PM-CMR, respectively (for details see~\Cref{appen:real}). As shown in~\Cref{fig:real_dataset}, DML-CMR and CE-DML-CMR demonstrate comparable, if not lower, MSE of fitting $\widehat{f}$ than the other methods. This shows that our algorithm is reliable when dealing with real-world data.

\begin{figure}[t]
\centering
\begin{subfigure}[c]{0.4\textwidth}
\centering
\includegraphics[width=1\textwidth]{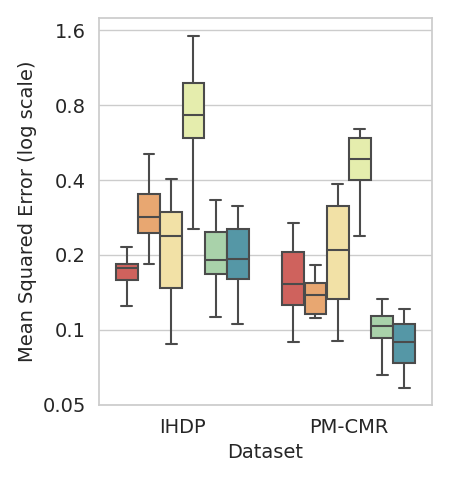}
\end{subfigure}
\begin{subfigure}[c]
{0.3\textwidth}
\centering
\includegraphics[width=0.6\textwidth]{figures/legend_iv.png}
\end{subfigure}
\caption{The mean squared error of $\widehat{f}$ on the real-world datasets IHDP and PM-CMR for the IV regression task.}
\label{fig:real_dataset}
\end{figure}

\subsection{Proximal Causal Learning}
For the PCL task (details in~\Cref{appen:pcl}), we compare our methods with PCL methods CEVAE~\citep{Im2021}, PMMR~\citep{Mastouri2021}, KPV~\citep{Mastouri2021}, DFPV~\citep{Xu2021}, NMMR U~\citep{Kompa2022}, NMMR V~\citep{Kompa2022} and PKDR~\citep{Wu2024}. We also use DNN estimators for both stages, with network architecture and hyperparameters provided in~\Cref{appen:networks}.

\subsubsection{Ticket Demand Dataset}

\begin{figure*}[t]
\centering
\begin{subfigure}[c]
{0.65\textwidth}
\centering
\includegraphics[width=0.9\textwidth]{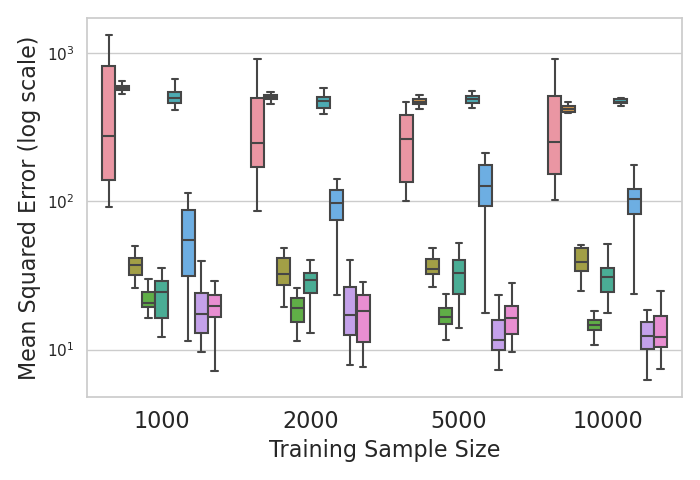}
\end{subfigure}
\begin{subfigure}[c]
{0.3\textwidth}
\centering
\includegraphics[width=0.6\textwidth]{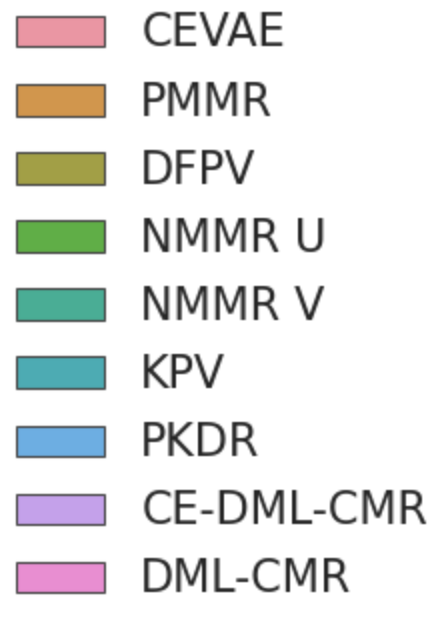}
\end{subfigure}
\caption{The mean squared error of $\widehat{f}$ on the ticket demand dataset for the PCL task.}
\label{fig:pcl_demand}
\end{figure*}

Similarly to IV regression, we start with the ticket demand dataset~\citep{Hartford2017DeepPrediction}, which has been adapted to the PCL setting~\citep{Xu2021}. We aim to understand how ticket prices affect ticket sales and learn the causal function $f_0$. The hidden confounder in this case is the varying demand $U$, while the cost of fuel $V$ is the treatment proxy, which directly impacts the ticket price, and the number of views on the airline's reservation website $W$ is the outcome proxy. Details of this dataset are included in~\Cref{appen:pcl_demand}.

The results for learning $f_0$ with this dataset of various sizes are provided in~\Cref{fig:pcl_demand}. It can be seen that DML-CMR and CE-DML-CMR achieved state-of-the-art performance with very similar performance to each other. NMMR U can match DML-CMR at smaller dataset sizes, but DML-CMR achieves lower MSE at 5000 and 7500 sample sizes. The performance gap between CE-DML-CMR and DML-CMR is small in this case, as expected, because the variables are low-dimensional. 

\subsubsection{High-dimensional dSprites Dataset}

\begin{figure*}[t]
\centering
\begin{subfigure}[c]
{0.65\textwidth}
\centering
\includegraphics[width=0.9\textwidth]{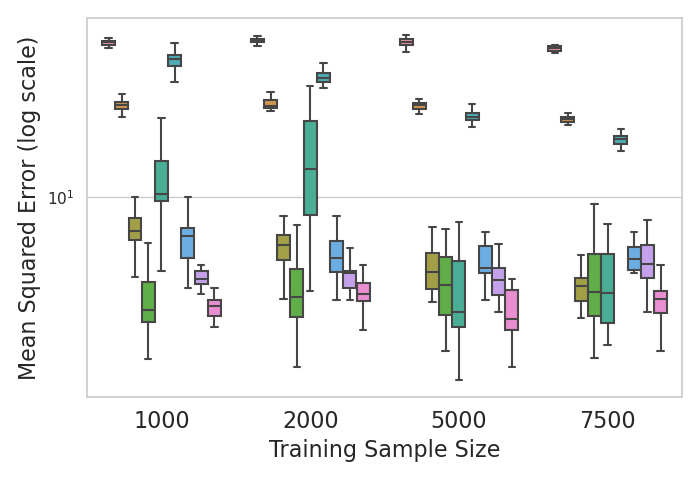}
\end{subfigure}
\begin{subfigure}[c]
{0.3\textwidth}
\centering
\includegraphics[width=0.6\textwidth]{figures/pcl/legend_pcl.png}
\end{subfigure}
\caption{The mean squared error of $\widehat{f}$ on the dSprites dataset with high dimensional treatment for the PCL task.}
\label{fig:pcl_dsprite}
\end{figure*}

For a high-dimensional dataset, we adopt the dSprites dataset~\citep{dsprites17} for PCL, first introduced by~\citet{Xu2021}. dSprites is an image (64 $\times$ 64) dataset where each image is described by five parameters: \textit{shape}, \textit{scale}, \textit{rotation}, \textit{posX} and \textit{poxY}. The treatments are these high-dimensional dSprites images, the hidden confounder is posY, the proxies are noisy observations of scale, rotation, and posX, and the outcome is defined by a nonlinear causal function. Details of this dataset are provided in~\Cref{appen:pcl_highdim}.

The results are presented in~\Cref{fig:pcl_dsprite}. DML-CMR achieved similar performance to the state-of-the-art methods while outperforming CE-DML-CMR. In addition, a lower variance can be observed when using DML-CMR compared to other methods, especially for smaller data sizes. The NMMR methods in some cases outperform CE-DML-CMR. This is in line with our previous observations that, without k-fold validation, performance can be worse for high-dimensional datasets and the full debiasing effect of the DML framework is required to achieve the best results.
\section{Conclusion}\label{sec:conclusion}

We have proposed a novel estimator for solving CMRs, DML-CMR. Using the DML framework and our novel Neyman orthogonal score, DML-CMR can effectively estimate solutions to CMR problems with fast convergence rate guarantees by mitigating the regularisation and overfitting biases in two-stage estimations. We theoretically analysed DML-CMR and proved a convergence rate of $O(N^{-1/2})$ with high probability under mild regularity and parametric assumptions. We also demonstrated interesting connections between the notion of ill-posedness for CMRs and DML's identifiability condition. We applied DML-CMR to problems in causal inference such as IV regression and proximal causal learning, and evaluated it on corresponding benchmarks, including semi-synthetic real-world data. Our experiments demonstrated that DML-CMR achieves state-of-the-art performance against similar algorithms that are developed specifically for the IV regression and PCL problems, as well as general CMR solvers, with lower estimation error and better stability. 

Future work includes considering other estimation methods for the nuisance parameters following our Neyman orthogonal score, and theoretically analysing our Neyman orthogonal score for estimating nonparametric functions of interest following~\citep{Foster2019OrthogonalLearning}.

\section*{Acknowledgments and Disclosure of Funding}
This work was supported by the EPSRC Prosperity Partnership FAIR (grant number EP/V056883/1). DS acknowledges funding from the Turing Institute and Accenture collaboration. AS was partially supported by the National Research Foundation Singapore and DSO National Laboratories under the AI Singapore Programme AISG Award No: AISG2-RP-2020-018, and by the Office of Naval Research (ONR) grant N00014-24 -1-2470.  MK and FQ acknowledge funding from ELSA: European Lighthouse on Secure and Safe AI project (grant agreement
No. 101070617 under UK guarantee). MK receives funding from the ERC under the European Union’s Horizon 2020 research and innovation programme (\href{http://www.fun2model.org}{FUN2MODEL}, grant agreement No.~834115).

\newpage

\appendix

\section{Computationally Efficient CE-DML-CMR}\label{appen:cedml}
\begin{algorithm}[ht]
   \caption{Computationally Efficient CE-DML-CMR}
   \label{alg:ce-dml}
\begin{algorithmic}
   \STATE {\bfseries Input:} Dataset $\dataset$ with size $N$, mini-batch size $n_b$
   \STATE Learn $\widehat{s}$ and $\widehat{g}$ using $\dataset$
   \STATE Initialise $f_{\widehat{\theta}}$
   \REPEAT
   \STATE Sample $n_b$ data $c_i$ from $\dataset$
   \STATE $\mathcal{L}=\widehat{\expectE}_{c_i}\left[(\widehat{s}(c)-\widehat{g}(f_\theta,c))^2\right]$
   \STATE Update $\widehat{\theta}$ to minimise loss $\mathcal{L}$ 
    \UNTIL{convergence}
    \STATE {\bfseries Output:} The CE-DML-CMR estimator $f_{\widehat{\theta}}$
\end{algorithmic}
\end{algorithm}

The standard DML-CMR with $K$-fold cross-fitting trains $\widehat{s}$ and $\widehat{g}$ $K$ times on different subsets of the dataset to tackle overfitting bias, but is computationally expensive. Therefore, as mentioned in~\Cref{sec:exp}, we also evaluate CE-DML-CMR, a computationally efficient version of DML-CMR that does not apply $K$-fold cross-fitting and trains $\widehat{s}$ and $\widehat{g}$ only once using the entire dataset. It uses the same Neyman orthogonal score as the standard DML-CMR, so it still enjoys the partial debiasing effect~\citep{Mackey2018} from the Neyman orthogonal score. However, without $K$-fold cross-fitting, it lacks the theoretical convergence rate guarantees provided by~\Cref{thm:dml}. CE-DML-CMR can be viewed as a trade-off between computational complexity and theoretical guarantees, and we found that CE-DML-CMR empirically performs as well as standard DML-CMR on low-dimensional datasets, where overfitting bias is not prevalent.

\section{The Score Function for Standard Two-Stage CMR Estimators}\label{appen:score}

In this section, we show that the learning objective, or score function, for standard two-stage CMR estimators~\citep{Angrist1996,Hartford2017DeepPrediction,Singh2019} is not Neyman orthogonal and thus cannot be used to create a DML estimator for the CMR problem.

\begin{repeatprop}{prop:standard_score}
The score (or objective) function for standard two-stage CMR estimators
$\ell=(Y-\widehat{g}(f,c))^2$
is not Neyman orthogonal at $(f_0, g_0)$.
\end{repeatprop}

\begin{proof}
The score $\ell=(Y-\widehat{g}(f,c))^2$ is not Neyman orthogonal because, first of all, $\expectE[(Y-g_0(f_0,c))^2]=\expectE[(Y-\expectE[Y\lvert C])^2]\neq0$ since $\expectE[f_0(X)\lvert C]=\expectE[Y\lvert C]$ and $Y-\expectE[Y\lvert C]\neq0$ due to the noise on $Y$. This violates the basic condition for a Neyman orthogonal score that the score function equals zero with the true functions $f_0$ and $g_0$.

Secondly, the Gateaux derivative against small changes in $g$ for score $\expectE[(Y-g_0(f_0,c))^2]$ at $(f_0, g_0)$ is
\begin{align*}
\frac{\partial}{\partial r}\expectE\Bigl[&(Y-g_0(f_0,C)-r\cdot g(f_0,C))^2\Bigr]\\
=&\frac{\partial}{\partial r}\expectE\Bigl[(Y-g_0(f_0,C))^2-2r\cdot(Y-g_0(f_0,C))g(f_0,C)+r^2\cdot g(f_0,C)^2\Bigr]\\
=&\expectE\Bigl[2(Y-g_0(f_0,C))g(f_0,C)+2r\cdot g(f_0,C)^2
\Bigr],
\end{align*}
and, when $r=0$, this derivative evaluates to
\begin{equation*}
\expectE[2(Y-g_0(f_0,c))g(f_0,c)]=\expectE[2
(Y-\expectE[Y\lvert C])g(f_0,c)]
\end{equation*}
which does not equal to 0 for general $g\in\mathcal{G}$ since generally $g(f_0,c)$ and the residual $(Y-\expectE[Y\lvert C])$ are correlated. Therefore, this standard score function for two-stage CMR estimation is not Neyman orthogonal at $(f_0, g_0)$.
\end{proof}

\section{Proofs}
In this section, we restate all the conditions required to prove the $N^{-1/2}$ convergence rate guarantees for the DML-CMR estimator, and provide the omitted proofs in the main paper for \Cref{thm:neyman}, Lemma~\ref{lemma:nuisances}, \Cref{thm:dml} and Corollary~\ref{coro:function_convergence}.

\subsection{DML-CMR $N^{-1/2}$ Convergence Rate Guarantees}\label{appen:dml}

To obtain $N^{-1/2}$ convergence rate guarantees of the DML-CMR estimator, the following conditions must be satisfied.

\noindent\textbf{Condition~\ref{condition:dml} [Conditions for $N^{-1/2}$ convergence of DML, Assumption 3.3 and 3.4 in~\citet{Chernozhukov2018Double/debiasedParameters}]}

For sample size $N\geq3$:
\begin{enumerate}[label=(\alph*)]
\item The map $(\theta,(s,g))\mapsto \expectE[\orthoM(\dataset;f_{\theta},(s,g))]$ is twice continuously Gateaux-differentiable.
\item  The score $\orthoM$ obeys the Neyman orthogonality conditions.
\item The true parameter $\theta_0$ obeys $\expectE[\orthoM(\dataset;f_{\theta_0},(s_0,g_0))]=0$ and $\Theta$ contains a ball of radius $c_1 N^{-1/2}\log N$ centered at $\theta_0$. 
\item For all $\theta\in\Theta$, the identification relationship
\begin{align*}
2\norm{\expectE[\orthoM(\dataset;f_{\theta},(s_0,g_0))]}\gtrsim \norm{J_0(\theta-\theta_0)}
\end{align*}
is satisfied, where $J_0\coloneqq\partial_{\theta^\prime}\{\expectE[\orthoM(\dataset;f_{\theta^\prime},(s_0,g_0))]\}|_{\theta^\prime=\theta_0}$ is the Jacobian matrix, with singular values bounded between $c_0>0$ and $c_1>0$.
\item  All eigenvalues of the matrix $\expectE[\orthoM(\dataset;f_{\theta_0},(s_0,g_0))\orthoM(\dataset;f_{\theta_0},(s_0,g_0))^T]$ are strictly positive (bounded away from zero).
\item Let $K$ be a fixed integer. Given a random partition $\{I_k\}_{k=1}^K$ of indices $[N]$, each of size $n=N/K$, the nuisance parameter estimator $\widehat{s}_k$ and $\widehat{g}_k$ learnt using data with indices $I^c_k$ belongs to shrinking realisation sets $\mathcal{S}_N$ and $\mathcal{G}_N$, respectively, and the nuisance parameters should be estimated at the $o(N^{-1/4})$ rate, e.g., $\norm{\widehat{s}-s_{0}}_2=o(N^{-1/4})$.
\end{enumerate}

To formalise the convergence rate guarantees in relationship to the technical conditions, we have the following proposition as a direct result of Theorem 3.3 of~\citet{Chernozhukov2018Double/debiasedParameters}.

\begin{proposition}[Theorem 3.3 of~\citet{Chernozhukov2018Double/debiasedParameters}
]\label{prop:thm3.3}
If all conditions in~\Cref{condition:dml} hold, then the DML estimator $\widehat{\theta}$ as defined in~\Cref{defn:dml} is concentrated in a $1/\sqrt{N}$ neighbourhood of $\theta_0$, and is approximately linear and centered Gaussian:
\begin{align*}
    \frac{\sqrt{N}}{\sigma}(\widehat{\theta}-\theta_0)=\frac{1}{\sqrt{N}}\sum{\bar{\orthoM}(\dataset_i)+O(\rho_N)}\rightarrow \mathcal{N}(0,1) \text{ in distribution},
\end{align*}
where $\bar{\orthoM}(\cdot)\coloneqq -\sigma^{-1}J_0^{-1}\orthoM(\cdot,\theta_0,\eta_0)$ is the influence function, $J_0$ is the Jacobian of $\orthoM$, the approximate variance is $\sigma^2 \coloneqq J_0^{-1}\expectE[\orthoM(\dataset,\theta_0,\eta_0)\orthoM(\dataset,\theta_0,\eta_0)^T](J_0^{-1})^T$, and the size of the remainder $\rho_N$ converges to 0.
\end{proposition}
\begin{proof}
This is a direct consequence of Theorem 3.3 from~\citet{Chernozhukov2018Double/debiasedParameters}, which states the convergence properties of the DML estimator for nonlinear score functions. Assumptions 3.3 and 3.4 in~\citet{Chernozhukov2018Double/debiasedParameters} required for Theorem 3.3 to hold are contained in~\Cref{condition:dml}. 
\end{proof}

We will show that all of these conditions are satisfied in the proof of~\Cref{thm:dml}. To begin with, we prove~\Cref{thm:neyman}, which shows that our score function $\orthoM$ is Neyman orthogonal.

\begin{repeatthm}{thm:neyman}[Neyman orthogonality]
The score function $\orthoM(\dataset;f,(s,g))=(s(c)-g(f,c))^2$ obeys the Neyman orthogonality conditions at $(f_0,(s_0,g_0))$.
\end{repeatthm}
\begin{proof}\label{appen:neyman}
Firstly, by~\Cref{eq:h_exp}, we have $s_0(C)=g_0(f_0,C)$, thus
\begin{equation*}
\orthoM(\dataset;f_0,(s_0,g_0))=\expectE\Bigl[(s_0(C)-g_0(f_0,C))^2\Bigr]=0
\end{equation*}
Then we compute the derivative w.r.t. small changes in the nuisance parameters. For all $s,g\in \mathcal{S}, \mathcal{G}$,
\begin{align*}
\frac{\partial}{\partial r}&\expectE\Bigl[(s_0(C)+r\cdot s(C)-g_0(f_0,C)-r\cdot g(f_0,C))^2\Bigr]\\
=&\frac{\partial}{\partial r}\expectE\Bigl[2r(s_0(C)-g_0(f_0,C))(s(C)-g(f_0,C))+r^2(s(C)-g(f_0,C))^2\Bigr]\\
=&\expectE\Bigl[2(s_0(C)-g_0(f_0,C))(s(C)-g(f_0,C))+2r(s(C)-g(f_0,C))^2\Bigr],
\end{align*}
and, when at $r=0$, the derivative evaluates to
\begin{align*}
    \expectE\Bigl[2(s_0(C)-g_0(f_0,C))(s(C)-g(f_0,C))\Bigr] = \expectE\Bigl[0 \times(s(C)-g(f_0,C))\Bigr] = 0 \quad \forall s,g\in \mathcal{S}, \mathcal{G},
\end{align*}
since $s_0(C)=\expectE[Y\lvert C]=\expectE[f_0(X)\lvert C]=g_0(f_0,C)$. Therefore, our moment function $\orthoM$ is Neyman orthogonal at $(f_0,(s_0,g_0))$.
\end{proof}

In turn, we state and prove the convergence of the nuisance parameters $\widehat{s}$ and $\widehat{g}$ with the machinery developed for the analysis of the excess risk in constrained ERM in \Cref{appen:critical_radius}.

\begin{repeatlemma}{lemma:nuisances}[Convergence of nuisance parameters]
    Under Assumption~\ref{assump:dml}, let $\mathcal{S}^*_N$ be the star-hull of the realisation set $\mathcal{S}_N$ of function class $\mathcal{S}$,
    \[
    \mathcal{S}^*_N = \mleft\{ C \mapsto \gamma(s(C) - s_0(C)) \;:\; s \in \mathcal{S}_N, \gamma \in [0, 1] \mright\},
    \]
    $\mathcal{P}^*_N$ be the star-hull of the realisation set $\mathcal{P}_N$ of the function class $\mathcal{P}$,
    \[
    \mathcal{P}^*_N = \mleft\{ C \mapsto \gamma(F(\cdot|C) - F_0(\cdot|C)) \;:\; F \in \mathcal{P}_N, \gamma \in [0, 1] \mright\},
    \]
    and $\mathcal{G}^*_N$ be the star-hull of the realisation set $\mathcal{G}_N$ of the function class $\mathcal{G}$,
    \[
    \mathcal{G}^*_N = \mleft\{C, f \mapsto\gamma(g(C,f)-g_0(C,f))\;:\; g\in\mathcal{G}_N,  \gamma\in[0,1]\mright\},
    \]
     where $\mathcal{S}_N$, $\mathcal{P}_N$ and $\mathcal{G}_N$ are properly shrinking neighbourhoods of the true functions $s_0$, $F_0$ and  $g_0$. Then, there exist universal constants $c_1$ and $c_2$, for which we have that with probability at least $1 - \xi$, the estimation errors are bounded as
    \[
    \norm{\widehat{s}-s_0}_2^2 & \leq c_1 \mleft( \delta_N(\mathcal{S}_N^*)^2+\sqrt{\frac{\log(1/\zeta)}{N}} 
    \mright) ; \\
     \norm{\widehat{g} - g_0}_2^2 & \leq c_2 \mleft( \delta_N(\mathcal{P}^*_N)^2+\sqrt{\frac{\log(1/\zeta)}{N}} \mright).
    \]
\end{repeatlemma}
\begin{proof}
    The bound on estimation error for $\widehat{s}$ is straightforward and follows directly by \Cref{theorem:erm_excess}. In order to prove that, we only need to show that the choice of loss function for the ERM estimator $\widehat{s}$ is Lipschitz in its first argument. To this end, we formulate the ERM as
    \[
    \widehat{s} = \argmin_{s \in \mathcal{S}} \mathcal{L}_N (s), \quad \textup{where } \mathcal{L}_N(s) = \frac{1}{N} \sum_{i=1}^N \ell(s; y_i, c_i) \textup{ and } \ell(s; y, c) = \mleft(y - s(c)\mright)^2. 
    \]
    Thus, we have,
    \[
    \mleft| \ell(s_2; y, c_2) - \ell(s_1; y, c_1) \mright| & =  \mleft| \mleft(y - s_2(c_2)\mright)^2 - \mleft(y - s_1(c_1)\mright)^2 \mright| \\
    & = \mleft| \mleft(2 y + s_2(c_2) + s_1(c_1) \mright) \mleft( s_2(c_2) - s_1(c_1) \mright) \mright| \\
    & \leq 4 B \mleft| s_2(c_2) - s_1(c_1) \mright| \numberthis{eq:erm_l2},
    \]
    where we used the fact that $y, \| s \|_\infty \leq B$ by Assumption~\ref{assump:dml}. Thus, $\ell(s; y,c)$ is $4B$-Lipschitz in its argument and therefore, by \Cref{theorem:erm_excess}, for a universal constant $c_1$,
    \[
    \norm{\widehat{s}-s_0}_2^2 & \leq c_1 \mleft( \delta_N(\mathcal{S}_{N}^*)^2+\sqrt{\frac{\log(1/\zeta)}{N}} 
    \mright).
    \]
    For the estimation error of $\widehat{g}$, we recall that for any $F(\cdot|C) \in \mathcal{P}$, 
    \[g_{F}(f_\theta, c) \coloneqq \int f_\theta(x) F(x | C = c) dx.
    \]
    Thus, we can connect the estimation error of $\widehat{g}_{\widehat{F}}$ to the estimation error of the conditional density estimator $F$ by showing 
    \[
    \| \widehat{g} - g_0 \|_2 & = \| \widehat{g}_{\widehat{F}} - g_{F_0} \|_2 \\
    &\leq B \| \widehat{F} - F_0 \|_2 \numberthis{eq:est_P_est_G}
    \]
    To prove \eqref{eq:est_P_est_G}, we observe that for any $C$ and any test function $f\in\mathcal{F}$, we have that,
    \[
    \mleft| g_{F_1} (C, f) - g_{F_2} (C, f) \mright| & = \mleft| \int f(x) \mleft[ F_1 - F_2 \mright] (dx | C) \mright| \\
    & \leq \int \mleft| f(x) \mright|  \mleft| F_1 - F_2 \mright| (dx | C) \\
    & \leq B \int \mleft| F_1 - F_2 \mright| (dx | C), \numberthis{eq:temp_bound_on_g}
    \]
    since $\| f\|_\infty \leq B$ by Assumption~\ref{assump:dml}. Thus, by integrating w.r.t. joint law of $(C, f)$, and \eqref{eq:temp_bound_on_g}, we can show that \eqref{eq:est_P_est_G} holds since
    \[
    \| g_{F_1} - g_{F_2} \|_2^2 & = \mathbb{E}_{C} \mleft[ \mleft(g_{F_1}(C, f) - g_{F_2}(C, f) \mright)^2 \mright] \\
    & \leq B^2 \mathbb{E}_{C} \mleft[   \mleft(\int \mleft| F_1 - F_2 \mright| (dx | C) \mright)^2  \mright] \\
    & \leq B^2 \mathbb{E}_{C} \mleft[   \int \mleft[ F_1 - F_2 \mright]^2  (dx | C) \mright] \numberthis{eq:temp_bound_on_g_2}
    \\
    & = B^2 \| F_1 - F_2 \|_2^2,
    \]
    for any $F_1,F_2\in\mathcal{P}$,
    where in \eqref{eq:temp_bound_on_g_2}, we used Cauchy–Schwarz in $dx$, i.e.,
    \[
    \mleft( \int \mleft| h \mright| \mright)^2 \leq \int h^2.
    \]
    Having \eqref{eq:est_P_est_G} at hand, it suffices to prove the upper bound on the estimation error of $\widehat{F}$. To this end, we observe that Squared CDF error,
    \[
    \ell(F; x, c) = [\mathbbm{1}_{\{X \leq x \;|\;C = c\} } - \mathscr{F}(x|c) ]^2,
    \]
    is Lipschitz by a similar argument to that for \eqref{eq:erm_l2}, where $\mathscr{F}(x|c)$ is the conditional CDF induced by the conditional density $F$. It is not hard to observe that \emph{clipped} versions of other losses for density estimation such as integrated squared error (ISE), negative log-likelihood, and Hellinger-squared also satisfy the Lipschitz condition in the first argument. Thus, by \Cref{theorem:erm_excess}, for a universal constant $c_3$,
    \[
    \norm{\widehat{F}-F_0}_2^2 & \leq c_3 \mleft( \delta_N(\mathcal{P}_{N^*})^2+\sqrt{\frac{\log(1/\zeta)}{N}} 
    \mright).
    \]
    Choosing $c_2 = B^2 c_3$ and \eqref{eq:est_P_est_G} completes the proof.
\end{proof}

Now, we are ready to prove~\Cref{thm:dml}, which is the main theorem that states the $N^{-1/2}$ convergence rate guarantees for our DML estimator.

\begin{repeatthm}{thm:dml}[Convergence of the DML estimator for CMRs]
Let $f_{\theta_0}\in\mathcal{F}$ be a solution that satisfies the CMRs in~\Cref{eq:cmr}, let $\orthoM$ be the Neyman orthogonal score defined in~\Cref{eq:neyman_score} and let $J_0\coloneqq\partial_{\theta^\prime}\{\expectE[\orthoM(\dataset;f_{\theta^\prime},(s_0,g_0))]\}|_{\theta^\prime=\theta_0}$ be the Jacobian matrix of $\expectE[\orthoM]$ w.r.t. $\theta$. Suppose that the upper bound of the critical radius $\delta_N=o(N^{-1/4})$, for $\widehat{s}$, $\widehat{g}$, and $J_0$ has bounded singular values. Then, if Assumption~\ref{assump:parameter} and \ref{assump:dml} hold, our DML estimator $f_{\widehat{\theta}}$ satisfies that $\widehat{\theta}$ is concentrated in a $N^{-1/2}$ neighbourhood of $\theta_0$, and is approximately linear and centred Gaussian:
\begin{align*}
    \sqrt{N}(\widehat{\theta}-\theta_0)\rightarrow \mathcal{N}(0,\sigma^2) \text{ in distribution},
\end{align*}
where the estimator variance is given by
\begin{equation*}
\sigma^2 \coloneqq J_0^{-1}\expectE[\orthoM(\dataset,\theta_0,(s_0,g_0))\orthoM(\dataset,\theta_0,(s_0,g_0))^T](J_0^{-1})^T,
\end{equation*}
which is constant w.r.t. $N$.
\end{repeatthm}

\begin{proof}\label{appen:dml_guarantee}
Following Proposition~\ref{prop:thm3.3}, we need to check whether, under Assumption~\ref{assump:dml}, all of Condition~\ref{condition:dml} for DML $N^{-1/2}$ convergence rate is satisfied. Condition (a) is satisfied since $(s-g)^2$ is twice continuously differentiable with respect to $s$ and $g$. Condition (b) is satisfied by Theorem~\ref{thm:neyman}. Condition (c) is satisfied since $f_{\theta_0}$ satisfies the CMRs and, from Theorem~\ref{thm:neyman}, we have that $\expectE[\orthoM(\dataset;f_{\theta_0},(s_0,g_0))]=0$. In addition, from Assumption~\ref{assump:parameter}, since the true parameter $\theta_0\in\Theta$ is in the interior of $\Theta$, $\Theta$ contains some neighbourhood centered at the true parameter $\theta_0$.

Condition (d) is a sufficient identifiability condition, which states that the closeness of the score function at point $\theta$ to zero implies the closeness of $\theta$
to $\theta_0$. This assumption is standard in conditional moment problems and implies that the \textit{ill-posedness} (see Definition~\ref{def:ill-posed}) of the CMR problem is bounded, as shown in~\Cref{sec:ill-posedness}. To check condition (d), we first point out that, under analytical assumptions for $s, g$, and $h$, we can write down first order Taylor series for the score function $\expectE[\orthoM(\dataset;f_{\theta},(s_0,g_0))]$ around the point $\theta_0$,
\begin{align*}
\expectE[\orthoM(\dataset;f_{\theta},(s_0,g_0))] = \expectE[\orthoM(\dataset;f_{\theta_0},(s_0,g_0))] + J_0 (\theta - \theta_0) + O(\norm{\theta - \theta_0}^2).
\end{align*}
Plugging in validity of the score function $\orthoM(\dataset;f_{\theta},(s_0,g_0))$, i.e.,  $\expectE[\orthoM(\dataset;f_{\theta_0},(s_0,g_0))] = 0$, we infer that
\begin{align*}
    \norm{\expectE[\orthoM(\dataset;f_{\theta},(s_0,g_0))]}\gtrsim \norm{J_0(\theta-\theta_0)}.
\end{align*}
Now for identifiability, we only need to check that $J_0 J_0^T$ is non-singular, which is guaranteed by bounded singular value of $J_0$ as stated in the Theorem.

Condition (e) is the non-degeneracy assumption for covariance of the score function $\orthoM(\dataset;f_{\theta},(s_0,g_0))$. By definition,
\begin{align*}
    \expectE[\orthoM(\dataset;f_{\theta},(s_0,g_0)) \orthoM(\dataset;f_{\theta},(s_0,g_0))^T] = \int \orthoM(\dataset;f_{\theta},(s_0,g_0)) \orthoM(\dataset;f_{\theta},(s_0,g_0))^T d\probP(\dataset).
\end{align*}
By trace trick, for each data point $\dataset$, the only eigenvalue of $\orthoM(\dataset;f_{\theta},(s_0,g_0)) \orthoM(\dataset;f_{\theta},(s_0,g_0))^T $ is $\norm{\orthoM(\dataset;f_{\theta},(s_0,g_0))}^2 \geq 0 $, with $\orthoM(\dataset;f_{\theta},(s_0,g_0))$ as the corresponding eigenvector. Therefore, $\expectE[\orthoM(\dataset;f_{\theta},(s_0,g_0)) \orthoM(\dataset;f_{\theta},(s_0,g_0))^T]$ is positive-definite if for each member $d$ of the support of $\probP$, which is the distribution of $\dataset$, there are at least as many eigenvectors of $d$ as the number of dimension of $\orthoM(\dataset;f_{\theta},(s_0,g_0))$, which is true in our setting as the co-domain of $\orthoM(\dataset;f_{\theta},(s_0,g_0))$ is $\realNumber$.

Condition (f) is satisfied since we have the critical radius $\delta_N=o(N^{-1/4})$, and together with Lemma~\ref{lemma:nuisances}, the nuisance parameters converge sufficiently quickly to ensure $\norm{\widehat{s}-s_0}_2\leq O(\delta_N+N^{-1/2})=O(o(N^{-1/4})+ N^{-1/2})=o(N^{-1/4})$ and similarly $\norm{\widehat{g}-g_0}_2\leq O(\delta_N+N^{-1/2})=o(N^{-1/4})$. 

Therefore, all the conditions in Condition~\ref{condition:dml} are satisfied, which concludes the proof by Proposition~\ref{prop:thm3.3}.
\end{proof}

\begin{repeatcoro}{coro:function_convergence}
Let $f_{\widehat{\theta}}$ be the DML estimator for CMRs. If all assumptions for~\Cref{thm:dml} hold and there exists a constant $L>0$ such that $\norm{f_\theta(x)-f_{\theta_0}(x)}_2\leq L \norm{\theta-\theta_0}_2$ for all $x\in\mathcal{X}$ and $\theta\in\Theta$,
then for all $\zeta\in(0,1]$, we have that
\begin{align*}
\norm{f_{\widehat{\theta}}-f_{\theta_0}}_2=O\left(L\sqrt{\frac{\ln(1/\zeta)}{N}}\right),
\end{align*}
with probability $1-\zeta$.
\end{repeatcoro}

\begin{proof}\label{appen:function_convergence}
From theorem~\ref{thm:dml}, we have that the parameters $\widehat{\theta}$ for our DML estimator $f_{\widehat{\theta}}$ learnt from a dataset of size $N$ satisfy $(\widehat{\theta}-\theta_0)\xrightarrow{d}\mathcal{N}(0,\sigma^2/N)$, where $\sigma^2$ is the DML estimator's variance. This means that, for all $\epsilon>0$ and $\zeta\in(0,1]$, there exists an integer $K>0$ such that for all $N\geq K$,
\begin{align*}
\probP\left(\norm{\widehat{\theta}-\theta_0}>\epsilon\right)\leq 1-\Phi\left(\epsilon\cdot \sqrt{N}/\sigma\right)+\zeta/2,
\end{align*}
where $\Phi$ is the CDF of a standard Gaussian distribution.
If we assume $L$ to be a constant such that $\norm{f_\theta(x)-f_{\theta_0}(x)}\leq L \norm{\theta-\theta_0}$ for all $x\in \mathcal{X}$ and $\theta\in\Theta$, we have that for all $\epsilon>0$ and $\zeta\in(0,1]$, there exists an integer $K>0$ such that for all $N\geq K$,
\begin{align*}
\probP\left(\norm{f_{\widehat{\theta}}(x)-f_{\theta_0}(x)}>L\cdot\epsilon\right) &\leq 1-\Phi(\epsilon\cdot \sqrt{N}/\sigma)+\zeta/2 \quad \forall x\in\mathcal{X},\\
\implies \probP\left(\norm{f_{\widehat{\theta}}(x)-f_{\theta_0}(x)}\leq L\cdot\epsilon\right) &\geq \Phi(\epsilon\cdot \sqrt{N}/\sigma)-\zeta/2 \quad \forall x\in\mathcal{X},
\end{align*}

Now, for any $\zeta\in(0,1]$, we can choose $\epsilon>0$ such that $\Phi(\epsilon\cdot \sqrt{N}/\sigma)=1-\zeta/2$ since $0.5\leq1-\zeta/2<1$, and by substituting $\epsilon$ out of the above equation, we have that

\begin{equation*}
\probP\left(\norm{f_{\widehat{\theta}}-f_{\theta_0}}_2\leq L\cdot \Phi^{-1}(1-\zeta/2)\sigma/\sqrt{N}\right)\geq 1-\zeta.
\end{equation*}

From Blair \textit{et al.}'s approximation for the inverse of the error function (erf)~\citep{Blair1976RationalFunction}, we have that, for all $y\in(0,1]$, $\Phi^{-1}(1-y)\leq\sqrt{-2\ln(y)}$. Thus, we conclude that there exists $K>0$ such that for all $N>K$,

\begin{align*}
\norm{f_{\widehat{\theta}}-f_{\theta_0}}_2\leq L\cdot \Phi^{-1}(1-\zeta/2)\sigma/\sqrt{N}&\leq L \sigma\sqrt{-2\ln(\zeta/2)}/\sqrt{N}\\
&= L \sigma\sqrt{2\ln(2/\zeta)}/\sqrt{N}\\
&= \sqrt{2} L \sigma \sqrt{\frac{\ln(2/\zeta)}{N}}\quad\text{ with probability }1-\zeta,
\end{align*}

which completes the proof.
\end{proof}

\subsection{Ill-posedness and DML Identification}\label{appen:illposed}

\begin{repeatprop}{prop:ill-posed}
For all $\theta\in\Theta$, if there exists a constant $L>0$ such that $\norm{f_\theta(x)-f_{\theta_0}(x)}_2\leq L \norm{\theta-\theta_0}_2$ for all $x\in\mathcal{X}$ and $\theta\in\Theta$, then Condition~\ref{condition:dml} (d), which states 
\begin{align*}
2\norm{\expectE[\orthoM(\dataset;f_{\theta},(s_0,g_0))]}\geq \norm{J_0(\theta-\theta_0)}
\end{align*}
and the Jacobian matrix $J_0$ have singular values bounded between $c_0>0$ and $c_1>0$, implies the ill-posedness is bounded by $\ill\leq L/\sqrt{c_0}$.
\end{repeatprop}

\begin{proof}
Recall that our score function is $\orthoM(\dataset;f_{\theta},(s,g))=(s(c)-g(f,c))^2$ where $\orthoM(\dataset;f_{\theta},(s_0,g_0))=(s_0(c)-g_0(f,c))^2=(\expectE[Y-f(X)\lvert C])^2$. Under a finite-dimensional parameterised setting, we have that from $2\norm{\expectE[\orthoM(\dataset;f_{\theta},(s_0,g_0))]}\gtrsim \norm{J_0(\theta-\theta_0)}$,

\begin{align}
\norm{\expectE[f_{\theta_0}(X)-f_\theta(X)\lvert C]}_2^2\nonumber
=&\norm{\expectE[(\expectE[f_{\theta_0}(X)-f_\theta(X)\lvert C])^2]}\nonumber\\
=&\norm{\expectE[(\expectE[Y-f_\theta(X)\lvert C])^2]}\nonumber\\
=&\norm{\expectE[\orthoM(\dataset;f_{\theta},(s_0,g_0))]}\nonumber\\
\geq& \frac{1}{2}\norm{J_0(\theta-\theta_0)}\nonumber\\
=&\sqrt{(\theta-\theta_0)^T(J_0^TJ_0)(\theta-\theta_0)}\nonumber\\
\geq& \frac{1}{2}\sqrt{c_0^2 \norm{(\theta-\theta_0)}_2^2}\geq \frac{1}{2}c_0 \norm{(\theta-\theta_0)}_2\geq \frac{1}{2}c_0 \norm{(\theta-\theta_0)}_2^2\label{eq:id_to_ill1}
\end{align}
for $\norm{(\theta-\theta_0)}\leq1$ and the singular value lower bound $c_0>0$ of $J_0$. With a local Lipschitz condition of $f_\theta$ around $\theta_0$: $\norm{f_\theta(x)-f_{\theta_0}(x)}_2\leq L\norm{\theta-\theta_0}_2$ for all $x\in\mathcal{X}$ and $\theta\in\Theta$, we have that 
\begin{align}
    \norm{f_{\theta_0}-f_\theta}_2^2&=\expectE[(f_{\theta_0}(x)-f_\theta(x))^2]\nonumber\\
    &\leq \expectE[(L\norm{\theta_0-\theta})^2]\nonumber\\
    &\leq L^2\norm{\theta_0-\theta}^2\nonumber\\
    \implies \norm{\theta_0-\theta}^2&\geq \frac{\norm{f_{\theta_0}-f_\theta}_2^2}{L^2}\label{eq:id_to_ill2}
\end{align}

Therefore, from~\Cref{eq:id_to_ill1} and~\Cref{eq:id_to_ill2}, we have that
\begin{align*}
\norm{\expectE[f_{\theta_0}(X)-f_\theta(X)\lvert C]}_2^2\geq c_0 \norm{(\theta-\theta_0)}_2^2&\geq \frac{c_0\norm{f_{\theta_0}-f_\theta}_2^2}{L^2}\\
\implies \norm{\expectE[f_{\theta_0}(X)-f_\theta(X)\lvert C]}_2\geq \sqrt{c_0} \norm{(\theta-\theta_0)}_2&\geq \frac{\sqrt{c_0}\norm{f_{\theta_0}-f_\theta}_2}{L}
\end{align*}

which bounds the ill-posedness by
\begin{align}
\ill=\sup_{f\in\mathcal{F}} \frac{\norm{f_{\theta_0}-f_\theta}_{2}}{\norm{\expectE[f_{\theta_0}(X)-f_\theta(X) \lvert C]}_{2}}\leq \frac{L}{\sqrt{c_0}}.
\end{align}
\end{proof}

\subsection{Constrained Empirical Risk Minimisation Bounds} \label{appen:critical_radius}

In this section, we introduce basic concepts from empirical process theory and further discuss bounds on the excess risk of general Empirical Risk Minimizer (ERM) in the style of \citet{Wainwright2019,Foster2019OrthogonalLearning}.

\begin{definition}
    The critical radius denoted by $\delta_N(\mathcal{H}^*)$ is defined as the minimum $\delta$ that satisfies the following upper bound on the local Gaussian complexity of a star-shaped function class $\mathcal{H}^*$\footnote{A function class 
  $\mathcal{H}$ is star-shaped if, for every $h \in \mathcal{H}$ and $\alpha \in [0, 1]$, we have $\alpha h \in \mathcal{H}$.}, $\mathcal{G(\mathcal{H}^*, \delta)} \leq {\delta^2}/2$, where local Gaussian complexity is defined as
    \begin{align*}
    \mathcal{G(\mathcal{H}^*, \delta)} = \expectE_{\epsilon}\left[\sup_{h \in \mathcal{H}^*: \norm{h}_N \leq \delta}  \langle \epsilon, h \rangle \right], \numberthis{eq;fixed_point}
    \end{align*}
    with $\epsilon$ being a random i.i.d. zero-mean Gaussian vector.
\end{definition}
The critical radius is a standard notion to bound the estimation error in the regression problem. Since local Gaussian complexity can be viewed as an expected value of a supremum of a stochastic process indexed by $g$, we can apply empirical process theory tools, namely the Dudley's entropy integral~\citep{Wainwright2019,van2014probability}, to provide a bound on the critical radius,
\begin{align*}
    \mathcal{G(\mathcal{H}^*, \delta)} \leq \inf_{\alpha \geq 0} \left \{\alpha + \frac{1}{\sqrt{N}}  \int_{\alpha/4}^{\delta} \sqrt{\log \mathcal{N}(\mathcal{H}^*, L^2(\mathbb{P}_N), \epsilon)}\:d\epsilon\right \},
\end{align*}
where $\mathcal{N}(\mathcal{H}^*, L^2(\mathbb{P}_N), \epsilon)$ is the $\epsilon$-covering number of the function class $\mathcal{H}^*$ in $L^2(\mathbb{P}_N)$ norm. Now, by placing $\alpha = 0$, when the integral is a single scale value of $ \sqrt{\log \mathcal{N}(\mathcal{H}^*, L^2(P_n), \epsilon)}$, we infer that
\begin{align*}
    \mathcal{G(\mathcal{H}^*, \delta)} \leq \frac{\delta}{\sqrt{N}} \sqrt{\log \mathcal{N}(\mathcal{F}^*, L^2(\mathbb{P}_N), \epsilon)}.
\end{align*}
Thus, the critical radius of $\mathcal{H}^*$ will be upper bounded by
\begin{align*}
    \delta_N(\mathcal{H}^*) \lesssim \frac{\sqrt{\log \mathcal{N}(\mathcal{H}^*, L^2(\mathbb{P}_N), \epsilon)}}{\sqrt{N}} = O(d_N (\mathcal{H}^*)^{1/2} N^{-1/2}),
\end{align*}
where \citet{Chernozhukov2022RieszNetForests,Chernozhukov2021AutomaticRegression} referred to 
\[
d_N (\mathcal{H}^*) \coloneqq \inf\mleft\{ d > 0: \log \mathcal{N}(\mathcal{H}^*, L^2(\mathbb{P}_N), \epsilon) \leq d \log \mleft(\frac{C}{\epsilon} \mright)\;\; \forall \epsilon \in (0, 1) \textrm{ and } C \textrm{ is a constant.}\mright\},
\]
as the effective dimension of the hypothesis space. Note that this rate matches the minimax lower bound of fixed design estimation for this setting~\citep{yang1999information}. 


Given the dataset $\mathcal{D} = \{z_i \in \mathcal{Z} \}_{i = 1}^N$ consisting of i.i.d. data points $z_i$ drawn from distribution $\mathbb{P}$, and a function class $\mathcal{H}$, we define the realisation of a function space by subscript $N$, e.g., $\mathcal{H}_N$ is the realisation of $\mathcal{H}$ in the $N$ observed data points. Since the definition of critical radius is for star-shaped function classes, we equip ourselves with the star-hull notation, where $\mathcal{H}^*_N$ is the star-hull of the function class $\mathcal{H}_N$ centred at the true function $h_0$, defined as
\[
    \mathcal{H}_N^* \coloneqq \mleft\{ Z \mapsto \gamma \mleft( h(X) - h_0 (Z) \mright)\; : \; f \in \mathcal{H}_N, \gamma \in [0, 1] \mright\},
\]
and denote its critical radius by $\delta_N(\mathcal{H}_N^*)$, or simply $\delta_N(\mathcal{H})$. In statistical learning, we are given a loss function $\ell: \mathcal{H} \times \mathcal{Z} \rightarrow \mathbb{R}$, and we have the risk and empirical risk defined accordingly as:
\[
\mathcal{L} (h) = \mathbb{E}_{z} [\ell(h; z)]  \quad \textrm{ and } \quad \mathcal{L}_N (h) = \frac{1}{N} \sum_{i = 1}^N \ell(h; z_i). 
\]
Let us denote the ground truth function by $h_0$, i.e., the minimizer of the risk,
\[
h_0 = \argmin_{h \in \mathcal{H}} \mathcal{L}(h).
\]
The ERM algorithm proposes the estimator $\widehat{h}$ that minimises the empirical risk of the observed $N$ data points,
\[
\widehat{h} = \argmin_{h \in \mathcal{H}} \mathcal{L}_N(h).
\]


\begin{theorem}[{\citet[Lemma 7]{Foster2019OrthogonalLearning}}] \label{thm:foster_emp}
    Consider a function class $\mathcal{H}$ and its star-hull $\mathcal{H}^*$, with $\sup_{h \in \mathcal{H}^*} \| h \|_\infty \leq 1$, critical radius $\delta_N(\mathcal{H}^*)$, and any choice of $\delta$ such that,
    \[ 
    \delta^2 \geq \max \mleft\{ \delta_N(\mathcal{H}^*)^2, \frac{4 \log (41 \log(2c_1 N))}{c_1 N} \mright\},
    \] 
    for a constant $c_1$. Moreover, assume that the loss function $\ell$ is $L$-Lipschitz in its first argument with respect to $\ell_2$ norm. Then there exist universal constants $c_2$ and $c_3$ such that with probability at least $ 1- c_2 \exp\{ c_3 N \delta^2\}$,
    \[
    \mleft| \mleft( \mathcal{L}_N(h) - \mathcal{L}_N(h_0) \mright) - \mleft( \mathcal{L}(h) - \mathcal{L}(h_0)\mright) \mright| \leq 18 L \delta 
 \mleft\{ \| h - h_0 \|_2 + \delta \mright\}, \qquad \forall h \in \mathcal{H}.
    \]
\end{theorem}

We mention an immediate corollary of \Cref{thm:foster_emp} for functions $h$ that are bounded by $B$, which follows by rescaling arguments.

\begin{corollary} \label{corr:emp_us}
    Consider a function class $\mathcal{H}$ and its star-hull $\mathcal{H}^*$, with $\sup_{h \in \mathcal{H}^*} \| h \|_\infty \leq B$, critical radius $\delta_N(\mathcal{H}^*)$, and any choice of $\delta$ such that,
    \[ 
    \delta^2 \geq \max \mleft\{ \delta_N(\mathcal{H}^*)^2, \frac{4 \log (41 \log(2c_1 N))}{c_1 N} \mright\},
    \] 
    for a constant $c_1$. Moreover, assume that the loss function $\ell$ is $L$-Lipschitz in its first argument with respect to $\ell_2$ norm. Then there exist universal constants $c_2$ and $c_3$ such that with probability at least $ 1- c_2 \exp\{ c_3 N \delta^2/\max\{1, B\}^2\}$,
    \[
    \mleft| \mleft( \mathcal{L}_N(h) - \mathcal{L}_N(h_0) \mright) - \mleft( \mathcal{L}(h) - \mathcal{L}(h_0)\mright) \mright| \leq \frac{18 L \delta}{\max\{1, B\}} 
 \mleft\{ \| h - h_0 \|_2 + \delta \mright\}, \qquad \forall h \in \mathcal{H}.
    \]
\end{corollary}
\begin{proof}
    If $B \leq 1$, the proof follows trivially from \Cref{thm:foster_emp}, thus let us assume otherwise. Define the $B$-scaled function class $\widetilde{\mathcal{H}}$,
    \[
    \widetilde{\mathcal{H}} \coloneqq \mleft\{ \widetilde{h} \;:\; B \widetilde{h} \in \mathcal{H} \mright\}.
    \]
    Then, the loss function,
    \[
    \ell(h, z) = \ell(B \widetilde{h}, z), 
    \]
    is $LB$-Lipschitz in $\widetilde{h}$ and by homogeneity of the local Gaussian averages, i.e.,
    \[
    \mathcal{G}(\widetilde{\mathcal{H}}^*, r) = \frac{1}{B} \mathcal{G}(\mathcal{H}^*, B r),
    \]
     wee see that $\widetilde{\delta} \coloneqq \delta / B$ satisfies the fixed point condition~\eqref{eq;fixed_point}. Putting these pieces together and invoking \Cref{thm:foster_emp} for $\widehat{\mathcal{H}}$ completes the proof.
\end{proof}

We can equivalently write Corollary~\ref{corr:emp_us} in the failure probability format. That is, for a target failure probability $0 < \xi < 1$, define 
\[
\delta(\xi) \coloneqq \delta_N(\mathcal{H^*}) + \max\{1, B\} \sqrt{\frac{1}{c_3 N} \log (1/\xi)},
\]
then with probability at least $1 - \xi$,
\[
\mleft| \mleft( \mathcal{L}_N(h) - \mathcal{L}_N(h_0) \mright) - \mleft( \mathcal{L}(h) - \mathcal{L}(h_0)\mright) \mright| \leq \frac{18 L \delta(\xi)}{\max\{1, B\}} 
 \mleft\{ \| h - h_0 \|_2 + \delta(\xi) \mright\}, \qquad \forall h \in \mathcal{H}. \numberthis{eq:emp_failure}
\]

Now, we are ready to state and prove the following master theorem for the excess risk of the constrained ERM. For the analysis, in line with \citet{Chernozhukov2021AutomaticRegression}, we require that the population risk has positive curvature for identifiability purposes. Then, the generalisation bound in terms of excess risk can be converted into estimation error.

\begin{theorem}[Estimation Error of Constrained ERM] \label{theorem:erm_excess}
    Assume that the population risk $\mathcal{L}$ has a positive curvature, i.e., for a positive number $\lambda$,
    \[
    \mathcal{L}(h) - \mathcal{L}(h_0) \geq \frac{\lambda}{2} \| h - h_0 \|_2^2 \qquad \forall h \in \mathcal{H}, \numberthis{eq:curve}
    \]
    and the loss function is bounded,
    \[
    | \ell(h; z) | \leq M,
    \]
    and $L$-Lipschitz in its first argument w.r.t. $\ell_2$ norm.
    Then, the solution to the ERM algorithm:
    \[
    \widehat{h} \coloneqq \argmin_{h \in \mathcal{H}} \mathcal{L}_N(h), \numberthis{eq:erm_min}
    \]
    has the following estimation error with probability at least $1 - \xi$,
    \[
    \| \widehat{h} - h_0 \|_2^2 \leq C \mleft[  \mleft(  \mathcal{L}(h^*) - \mathcal{L}(h_0) \mright) + \delta_N(\mathcal{H}^*)^2 + \sqrt{\frac{\log(1/\xi)}{N}} \mright].
    \]
    for a universal constant $C$, where
    \[
    h^* = \arginf_{h \in \mathcal{H}} \mathcal{L}(h).
    \]
    In addition, if $h_0$ is realisable, i.e., $h_0 \in \mathcal{H}$, then,
    \[
        \| \widehat{h} - h_0 \|_2^2 \leq C \mleft[ \delta_N(\mathcal{H}^*)^2 + \sqrt{\frac{\log(1/\xi)}{N}} \mright].
    \]
\end{theorem}
\begin{proof}
    First, we want to upper bound the population excess risk,
    \[
    \mathcal{L} (\widehat{h}) - \mathcal{L} (h_0),
    \]
    and then, by the curvature of $\mathcal{L}$, we can convert this bound into an upper bound on the estimation error $\| \widehat{h} - h_0 \|_2$.

    To begin with, we have,
    \[
    \mathcal{L} (\widehat{h}) - \mathcal{L} (h_0) & = \underbrace{{\Big( \mathcal{L}(\widehat{h}) - \mathcal{L}_N (\widehat{h}) \Big)}}_\textup{(I)} + \underbrace{\Big( \mathcal{L}_N (\widehat{h}) - \mathcal{L}_N(h^*) \Big) }_{\textup{(II)}} + \underbrace{\Big( \mathcal{L}_N (h^*) - \mathcal{L}(h_0) \Big)}_{\textup{(III)}}.
    \]
    By ERM algorithm~\eqref{eq:erm_min}, we know that,
    \[
    \mathcal{L}_N(\widehat{h}) \leq \mathcal{L}_N(h^*),
    \]
    Hence, $\textup{(II)} \leq 0$. For the term $\textup{(I)}$,
    \[
    \textup{(I)} & =  \mathcal{L}(\widehat{h}) - \mathcal{L}_N (\widehat{h}) \\
    & = (\mathcal{L} - \mathcal{L}_N) (\widehat{h} - h_0) + (\mathcal{L} - \mathcal{L}_N) (h_0),
    \]
    and for the term $\textup{(III)}$, we know that,
    \[
    \textup{(III)} & = \mathcal{L}_N (h^*) - \mathcal{L}(h_0) \\
    & = ( \mathcal{L}(h^*) - \mathcal{L}(h_0) ) + (\mathcal{L}_N - \mathcal{L}) (h^*) \\
    & = ( \mathcal{L}(h^*) - \mathcal{L}(h_0) ) + (\mathcal{L}_N - \mathcal{L}) (h^* - h_0) + (\mathcal{L}_N - \mathcal{L}) (h_0).
    \]
    Summing these terms from $(\textup{I}), (\textup{II})$ and $(\textup{III})$ yields,
    \[
    \mathcal{L} (\widehat{h}) - \mathcal{L} (h_0) & \leq ( \mathcal{L}(h^*) - \mathcal{L}(h_0) ) + \underbrace{(\mathcal{L} - \mathcal{L}_N) (\widehat{h} - h_0)}_{\epsilon_1(\widehat{h})} + \underbrace{(\mathcal{L}_N - \mathcal{L}) (h^* - h_0)}_{\epsilon_2}.
    \]
    From Corollary~\ref{corr:emp_us}, especially the formulation in \eqref{eq:emp_failure}, we know that with probability at least $1 - \xi/2$,
    \[
    \sup_{h \in \mathcal{H}} \mleft| (\mathcal{L}_N - \mathcal{L}) (h - h_0)\mright| \leq \frac{18 L \delta(\xi/2)}{\max\{1, B\}} \mleft( \|h - h_0 \|_2 + \delta(\xi/2) \mright),
    \]
    where,
    \[
    \delta(\xi/2) = \delta_N(\mathcal{H^*}) + \max\{1, B\} \sqrt{\frac{1}{c_3 N} \log (2/\xi)}. \numberthis{eq:delta_xi_half}
    \]
    Thus, we can upper bound $\epsilon_1 (\widehat{h})$ as,
    \[
    \epsilon_1 (\widehat{h}) & = (\mathcal{L} - \mathcal{L}_N) (\widehat{h} - h_0) \leq   \frac{18 L \delta(\xi/2)}{\max\{1, B\}} \mleft( \|\widehat{h} - h_0 \|_2 + \delta(\xi/2) \mright).
    \]
    To bound the term $\epsilon_2$, define the random variable,
    \[
    X_i = \ell(h^*; z_i) - \ell(h_0; z_i),
    \]
    where $X_i$ is bounded as $|X_i| \leq 2 M$. Then, by Hoeffding's inequality, with probability at least $1 - \xi/2$, we infer that
    \[
    \epsilon_2 & = (\mathcal{L}_N - \mathcal{L}) (h^* - h_0)  \leq 2 \sqrt{2} M \sqrt{\frac{\log(4/\xi)}{N}}.
    \]
    Thus, by union bound and \eqref{eq:curve}, we conclude that with probability at least $1 - \xi$
    \[
    \frac{\lambda}{2} \| \widehat{h} - h_0 \|_2^2 & \leq ( \mathcal{L}(h^*) - \mathcal{L}(h_0) ) + \frac{18 L \delta(\xi/2)}{\max\{1, B\}} \mleft( \|\widehat{h} - h_0 \|_2 + \delta(\xi/2) \mright) + 2 \sqrt{2} M \sqrt{\frac{\log(4/\xi)}{N}} \\
    & = ( \mathcal{L}(h^*) - \mathcal{L}(h_0) ) + \frac{18 L \delta(\xi/2)}{\max\{1, B\}}  \|\widehat{h} - h_0 \|_2 + \frac{18 L \delta(\xi/2)^2}{\max\{1, B\}}  + 2 \sqrt{2} M \sqrt{\frac{\log(4/\xi)}{N}} \\
    & \leq ( \mathcal{L}(h^*) - \mathcal{L}(h_0) ) + \frac{\lambda}{4} \|\widehat{h} - h_0 \|_2^2
    + 
    \frac{1}{\lambda} \mleft( \frac{18 L \delta(\xi/2)}{\max\{1, B\}} \mright)^2  \\
    & \qquad \qquad  + \frac{18 L \delta(\xi/2)^2}{\max\{1, B\}} + 2 \sqrt{2} M \sqrt{\frac{\log(4/\xi)}{N}} \numberthis{eq:erm_temp1} \\
    \]
    where \eqref{eq:erm_temp1} follows by Young's inequality, i.e., for any positive number $\lambda > 0$,
    \[
    ax \leq \frac{\lambda}{4} x^2 + \frac{a^2}{\lambda}.
    \]
    Therefore, by cancelling out $\frac{\lambda}{4} \|\widehat{h} - h_0 \|_2^2$ from both sides,
    \[
    \| \widehat{h} - h_0 \|_2^2 \leq C_1 \mleft[ \frac{1}{\lambda} \mleft(  \mathcal{L}(h^*) - \mathcal{L}(h_0)  \mright) + \mleft( \frac{L^2}{\lambda \max\{1, B\}^2} + \frac{L}{\max\{1, B\}}\mright) \delta(\xi/2)^2 + M \sqrt{\frac{\log(1/\xi)}{N}} \mright],
    \]
    for a constant $C_1$.
    In turn, by \eqref{eq:delta_xi_half} and Young's inequality, we know that,
    \[
    \delta(\xi/2)^2 \leq C_2 \mleft[ \delta_N(\mathcal{H}^*)^2 + \frac{\max\{1, B\}^2}{N} \log(1/\xi) \mright],
    \]
    for a constant $C_2$. Hence,
    \[
    \| \widehat{h} - h_0 \|_2^2 & \leq C_3 \Bigg[ \frac{1}{\lambda} \mleft(  \mathcal{L}(h^*) - \mathcal{L}(h_0)  \mright) \\
    & \qquad + \mleft( \frac{L^2}{\lambda \max\{1, B\}^2} + \frac{L}{\max\{1, B\}}\mright) \mleft[ \delta_N(\mathcal{H}^*)^2 + \max\{1, B\}^2 \frac{ \log(1/\xi)}{N} \mright] \\
    & \qquad + M \sqrt{\frac{\log(1/\xi)}{N}} \Bigg] \\
    & \leq C_4 \Bigg[ \frac{1}{\lambda} \mleft(  \mathcal{L}(h^*) - \mathcal{L}(h_0)  \mright) + \frac{\max\{L^2, L\}}{\min\{\lambda, 1\} \max\{1 , B\}}  \delta_N(\mathcal{H}^*)^2 \\
    & \qquad + \frac{\max\{L^2, L\} \max\{1 , B\}}{\min\{\lambda, 1\}} \frac{\log(1/\xi)}{N} + M \sqrt{\frac{\log(1/\xi)}{N}} \Bigg], \numberthis{eq:erm_final}
    \]
    for constants $C_3$ and $C_4$.
    
    Therefore, by hiding dependence on the constants $\lambda, B, L$ and $M$ in a universal constant $C$, we can summarise the result~\eqref{eq:erm_final} and conclude the proof,
    \[
    \| \widehat{h} - h_0 \|_2^2 \leq C \mleft[  \mleft(  \mathcal{L}(h^*) - \mathcal{L}(h_0) \mright) + \delta_N(\mathcal{H}^*)^2 + \sqrt{\frac{\log(1/\xi)}{N}} \mright].
    \]
\end{proof}

\section{Examples of CMR Problems}\label{appen:apply_cmrs}

There are many concrete problems in statistical estimation, causal inference, and econometrics that are, in fact, CMR problems (see~\citet{Carrasco2007}, Section 1.3, for nine concrete CMR problems). In this section, we introduce in detail two CMR problems in causal inference, IV regression and proximal causal learning, which we evaluated experimentally in~\Cref{sec:exp}. To begin with, we provide a brief introduction of hidden confounders and structural causal models.

\subsection{Hidden Confounders}

\textit{Hidden confounders}~\citep{Pearl2000} are unobserved variables that influence both the \textit{actions} (or \textit{interventions}) and the \textit{outcome}. To properly account for these hidden confounders and understand the true causal effect of actions, we need to model the causal (or structural) relationship between the action and the outcome, which is expressed through a \textit{causal function}. However, learning the causal function in the presence of hidden confounders is known to be challenging and sometimes infeasible~\citep{Shpitser2008CompleteHierarchy}. To formalise the concept of hidden confounders and provide a framework for specifying the underlying causal
mechanisms in a data-generating process, we introduce structural causal models (SCMs).

\subsection{Structural Causal Model}

\begin{definition}[Structural Causal Model]
An SCM $M$ is a tuple $(U,V,F,P(U))$, where U is a set of exogenous (i.e. outside the model) random variables, which are
typically unobserved; V is a set of endogenous (i.e. inside the model) variables; $F=f_i$ is a set of deterministic functions where, for each $V_i\in V$, $f_i(pa_i,u_i)=v_i$ ($pa_i$ denotes the parent of $V_i$ and $U_i$ are exogenous variables linked to $V_i$). $P(U)$ is the joint distribution of exogenous variables.
\end{definition}

In this definition, endogenous variables are variables on which we would like to study the causal relationships (e.g., between price and revenue). Exogenous variables are external sources of noise (e.g., seasonality) that can confound the causal relationships between endogenous variables. Next, we introduce the concept of causal interventions, which are tools that allow us to study causal effects between variables. Interventions are defined through
a mathematical operator called $do(x)$~\citep{Pearl2000}. An intervention, denoted by $do(X=x)$, simulates a physical intervention by removing the natural dependencies of $X$ on its parent variables in the SCM and forcing it to take a specific value $x$, while keeping the rest of the model unchanged. The resulting causal model after the intervention is denoted $M_x$. The post-intervention distribution resulting from the
intervention $do(X = x)$ is given by the equation
\begin{align}
 \probP_M(y\lvert do(x)) = \probP_{M_x}(y),
\end{align}
where the post-intervention distribution of some variable $Y$ is defined as the distribution of $Y$ in the intervened model $M_x$.

For example, in a causal model where $A$ (treatment) affects $Y$ (outcome), an observational study may show correlation, but an intervention $do(A = a)$ would simulate a randomised experiment, ensuring that changes in $Y$ are due to $A$ and not other confounders.

\begin{figure}
    \centering
    \includegraphics[width=0.3\linewidth]{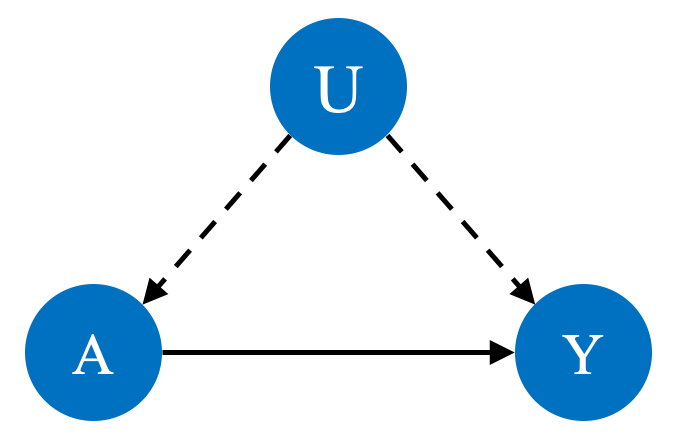}
    \caption{The causal graph of outcome $Y$, treatment $A$ and hidden confounder $U$.}
    \label{fig:causal_graph}
\end{figure}

In the SCM formulation, an exogenous random variable is considered a hidden confounder if it affects two or more endogenous variables, e.g., $V_i$ and $V_j$, and is unobserved. Consider a SCM that specifies two endogenous variables, the outcome $Y\in\mathcal{Y}$ and the treatment $A\in\mathcal{A}$: 
\begin{align}
    Y=f(A,U)\label{eq:scm},
\end{align}
where $U\in\mathcal{U}$ is a hidden confounder that affects both $A$ and $Y$, as illustrated in the causal graph depicted in~\Cref{fig:causal_graph}. Due to the presence of this hidden confounder, with only observational data, standard regressions (e.g., ordinary least squares) generally fail to produce consistent estimates of the causal relationship (also known as the average treatment effect) between $A$ and $Y$~\citep{Pearl2000}, i.e., $\expectE[Y \mid do(A)]$, where $do(\cdot)$ is the interventional operator. Therefore, the ability to identify the causal relationship between $A$ and $Y$ requires additional assumptions. Two classic techniques are IV regression~\citep{Newey2003} and proximal causal learning, each with the explicit assumption to observe additional variables in the model that help the identification of $\expectE[Y \mid do(A)]$, which we will introduce next.

\subsection{Instrumental Variables}\label{appen:iv}
\begin{figure}
    \centering
    \includegraphics[width=0.45\linewidth]{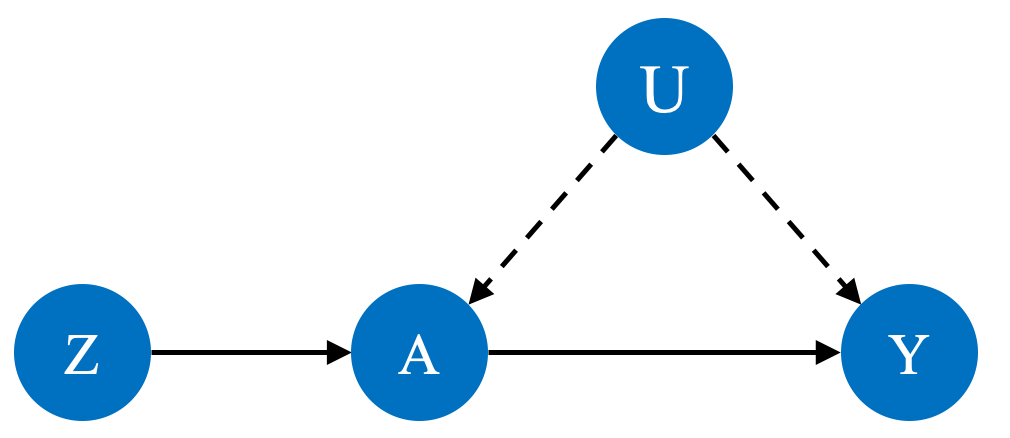}
    \caption{The causal graph of outcome $Y$, treatment $A$, hidden confounder $U$ and an instrumental variables $Z$.}
    \label{fig:causal_graph_iv}
\end{figure}
We first introduce the concept of Instrumental Variables (IVs). Under the SCM of outcome $Y\in\mathcal{Y}$, treatment $A\in\mathcal{A}$ and hidden confounder $U\in\mathcal{U}$ defined in~\Cref{eq:scm}, an IV $Z\in\mathcal{Z}$ is an observable variable that satisfies the following conditions~\citep{Newey2003}: 
\begin{itemize}[leftmargin=40pt, topsep=5pt]
\label{assump:iv}
    \item \textit{Unconfounded Instrument}: $Z\indep U$;
    \item \textit{Relevance}: $\probP(A\lvert Z)$ is not constant in $Z$;
    \item \textit{Exclusion}: $Z$ does not directly affect $Y$: $Z\indep Y \mid (A,U)$,
\end{itemize}
where a causal graph with an instrumental variable $Z$ is depicted in~\Cref{fig:causal_graph_iv}.

Furthermore, in order to identify the causal effect $\expectE[Y \mid do(A)]$, an additional assumption of \textit{additive noise} is required, where we assume that
\begin{align}
    Y=f(A)+\epsilon(U) \ \text{ with } \  \expectE[\epsilon(U)]=0\label{eq:iv_reg}.
\end{align}
Specifically, since the hidden confounder $U$ affects both $A$ and $Y$, it is generally the case that $\expectE[\varepsilon(U) \mid A]\neq 0$, which makes standard regression methods such as ordinary least squares fail to estimate the correct causal effect. The additive noise assumption in conjunction with the IV conditions is standard for the IV settings~\citep{Newey2003,Xu2020,Shao2024} and allows the minimal condition to identify the causal effect. Another observation we can make here is that, following~\Cref{eq:iv_reg}, 
\begin{align}
\expectE[Y \mid do(A)]&=\expectE[f(A)\mid do(A)]+\expectE[\epsilon(U)\mid do(A)]\\
&=f(A)+\expectE[\epsilon(U)]=f(A),
\end{align}
so the task of identifying the \textit{causal effect} $\expectE[Y \mid do(A)]$ is the same as learning the \textit{causal function} $f(A)$.

In order to identify $f(A)$, a key observation~\citep{Newey2003} is that,  by taking the expectation on both sides of~\Cref{eq:iv_reg} conditional on $Z$, we have
\begin{align}
\expectE[Y\lvert Z]&=\expectE[f(A)+\epsilon(U)\lvert Z\Big]\nonumber\\
&=\expectE[f(A)\lvert Z]+\expectE[\epsilon(U)]\nonumber\\
&=\expectE[f(A)\lvert Z]=\int f(A) \probP(A\lvert Z) dA,\label{eq:h_exp}
\end{align}
where the expectation $\expectE[Y\lvert Z]$ and the distribution $\probP(A\lvert Z)$ are both observable. Therefore, the problem of estimating the causal effect $\expectE[Y \mid do(A)]$ in the IV setting can be reduced to the CMR:
\begin{align}
\expectE[Y-f(A)\lvert Z]=0.
\end{align}

\subsection{Proximal Causal Learning}\label{appen:pcl}

\begin{figure}
    \centering
    \includegraphics[width=0.45\linewidth]{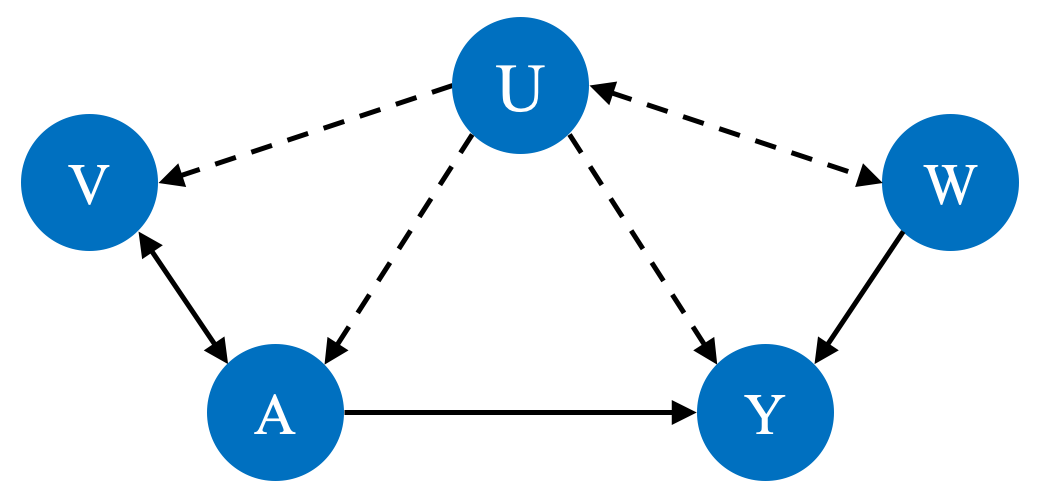}
    \caption{The causal graph of outcome $Y$, treatment $A$, hidden confounder $U$ and proxies $V$ and $W$}
    \label{fig:causal_graph_proxy}
\end{figure}

Next, we introduce proximal causal learning (PCL). Under the SCM of outcome $Y\in\mathcal{Y}$, treatment $A\in\mathcal{A}$ and hidden confounder $U\in\mathcal{U}$ defined in~\Cref{eq:scm}, PCL uses two proxy variables to identify the causal effect of the treatment $A$ on the outcome $Y$, i.e., $\expectE[Y \mid do(A)]$. The first proxy $V\in\mathcal{V}$ is a treatment-inducing proxy, and the second proxy $W\in\mathcal{W}$ is an outcome-inducing proxy. For $V$ and $W$ to be valid proxies, they need to satisfy the following \textit{conditional independence conditions}:
\begin{itemize}[leftmargin=40pt, topsep=5pt]
\label{assump:proxy}
    \item $Y\indep V \mid (A,U)$;
    \item $W\indep (A,V) \mid U$,
\end{itemize}
where a causal graph with proxies $V$ and $W$ is depicted in~\Cref{fig:causal_graph_proxy}.

In addition, for identifiability of the causal effect, proxies should satisfy \textit{completeness assumptions}. Let $l:\mathcal{U}\rightarrow \realNumber$ be any square integrable function, that is, $\norm{l}_2\leq \infty$. The following conditions hold for any $a\in \mathcal{A}$:
\begin{align}
    \expectE[l(U)\lvert A=a,W=w]=0 \quad\forall w\in \mathcal{W} &\iff l(u)=0 \text{ a.e. on } \probP(U)\\
    \expectE[l(U)\lvert A=a,V=v]=0 \quad\forall v\in \mathcal{V} &\iff l(u)=0 \text{ a.e. on } \probP(U)
\end{align}

It has been shown that when the conditional independence conditions and the completeness assumptions are satisfied, it is possible to identify $\expectE[Y\lvert do(A)]$ by solving CMRs.

\begin{proposition}[\citet{Miao2018}]
Let the conditional independence and completeness assumptions hold, then there exists at least one solution to the following CMR
\begin{align}
\expectE[Y\lvert A,V]&=\expectE[h(A,W)\lvert A,V]\label{eq:proximal}\\
&=\int h(A,W)\probP(A,W\lvert A,V) dW,
\end{align}
for all $(A,V)\in \mathcal{A}\times \mathcal{V}$. Let $h^*$ be a solution of~\Cref{eq:proximal}, then the \textit{causal effect} $\expectE[Y \mid do(A)]$ can be estimated by $\expectE_W[h(A,W)]$.
\end{proposition}

From this proposition, we can see that the problem of estimating the causal effect $\expectE[Y \mid do(A)]$ can be reduced to estimating $h^*$, which we denote as the bridge function following~\citet{Miao2018}. Therefore, estimating the causal effect in the PCL setting can be reduced to the CMR:
\begin{align}
\expectE[Y-h(A,W)\lvert A,V]=0
\end{align}

\begin{remark}
For both IV and PCL, it is possible to include additional observed confounders $X$ that affect both the treatment and the outcome as additional information or context. $X$ can also be confounded by $U$, and the resulting CMRs would be $\expectE[Y-f(A,X)\lvert Z,X]=0$ for IV regression and $\expectE[Y-h(A,W,X)\lvert A,V,X]=0$ for PCL.
\end{remark}

\section{Datasets Details} 
In this section, we provide details of the datasets considered in this paper for IV regression and proximal causal learning tasks.

\subsection{IV Regression}
We first provide the details for IV regression benchmarking datasets. Recall that we denote $A$ as the action, $Y$ as the outcome, $Z$ as the instrument, and $X$ as additional observed context and the CMR we are trying to solve is $\expectE[Y-f(A,X)\lvert Z,X]=0$.

\subsubsection{Ticket Demand Dataset}\label{appen:demand}

Here, we describe the aeroplane ticket demand dataset for IV regression, first introduced by~\citet{Hartford2017DeepPrediction}. The observable variables are generated by the following model:
\begin{align*}
r&=f_0((t,s),p)+\epsilon, \quad \expectE[\epsilon\lvert t,s,p]=0;\\
p&=25+(z+3)\psi(t)+\omega,
\end{align*}
where $r$ is the ticket sales (as the outcome variable $Y$) and $p$ is the ticket price (as the action variable $A$). $(t,s)$ are observed context variables, where $t$ is the time of year and $s$ is the customer type. The fuel price $z$ is introduced as an instrumental variable, which only affects the ticket price $p$. The noises $\epsilon$ and $\omega$ are correlated with correlation $\rho\in[0,1]$, where in our experiments we set $\rho=0.9$. $f_0$ is the true causal effect function, defined as
\begin{align*}
    f_0((t,s),p)&=100+(10+p)\cdot s \cdot \psi(t)-2p,\\
    \psi(t)&=2\left(\frac{(t-5)^4}{600}+\exp(-4(t-5)^2)+\frac{t}{10}-2\right),
\end{align*}
where $\psi(t)$ is a complex non-linear function of $t$ plotted in~\Cref{fig:nonlinear}. The offline dataset is sampled with the following distributions:
\begin{align*}
    s&\sim \text{Unif}\{1,...,7\}\\
    t&\sim \text{Unif}(0,10)\\
    z&\sim \mathcal{N}(0,1)\\
    \omega&\sim \mathcal{N}(0,1)\\
    \epsilon&\sim \mathcal{N}(\rho\omega,1-\rho^2).
\end{align*}
From the observations $(r,p,t,s,z)$, we estimate $\widehat{h}$ using IV regression methods, and the mean squared error between $\widehat{h}$ and the true causal function $f_0$ is computed on 10000 random samples from the above model. For the out-of-distribution test samples, we sample $t\sim \text{Unif}(1,11)$ instead.

We standardise the action and outcome variables $p$ and $r$ to centre the data around a mean of zero and a standard deviation of one following~\citet{Hartford2017DeepPrediction}. This is standard practice for DNN training, which improves training stability and optimization efficiency.

\begin{figure}[tb]
    \centering
\includegraphics[width=0.5\textwidth]{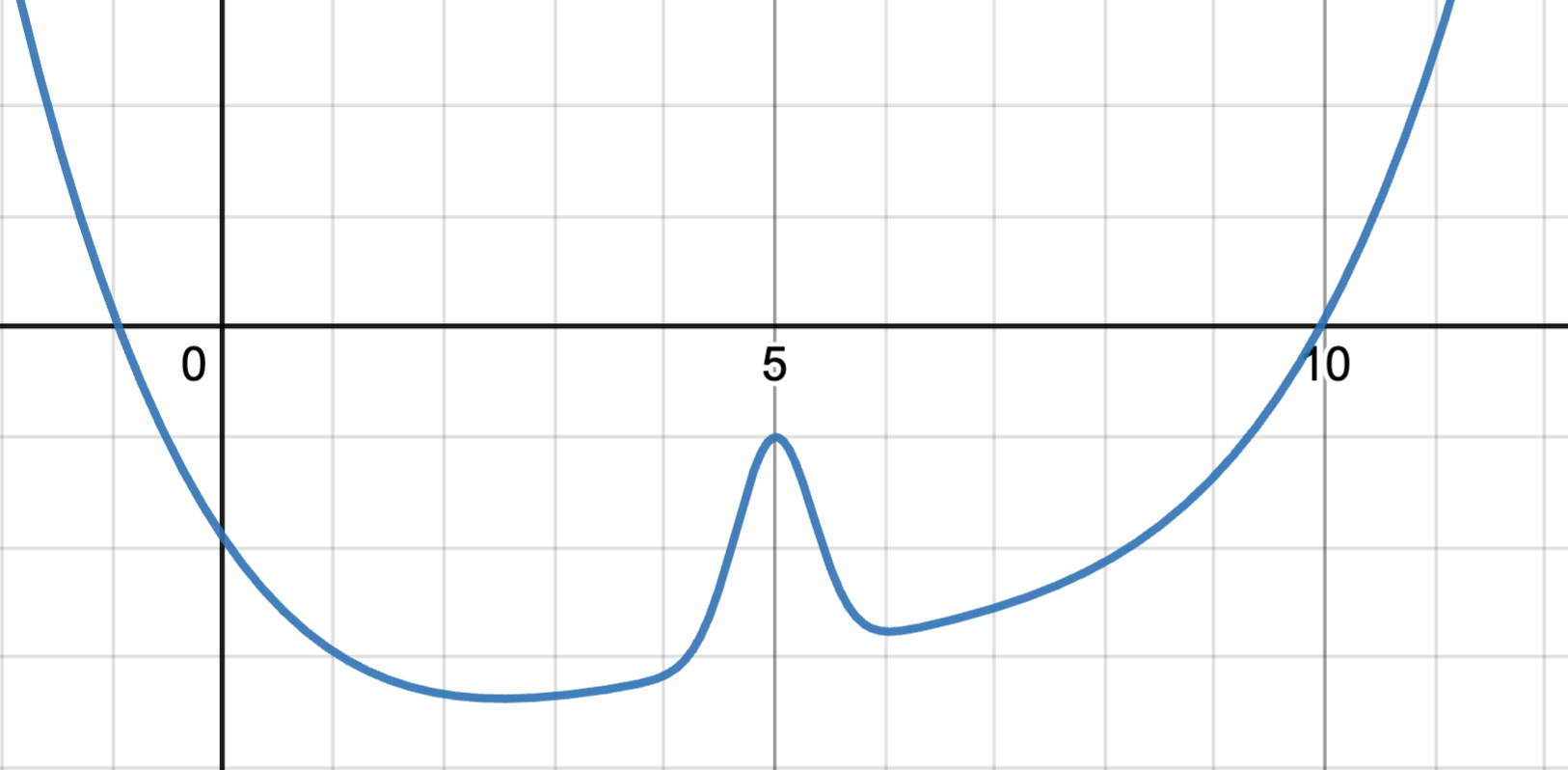}
    \caption{A graph of the nonlinear function $\psi(t)$ in the ticket demand dataset for IV regression.}
    \label{fig:nonlinear}
\end{figure}

\subsubsection{Ticket Demand High-Dimensional Setting}

For the high-dimensional setting, we again follow~\citet{Hartford2017DeepPrediction} to replace the customer type $s\in[7]$ in the low-dimensional setting with images of the corresponding handwritten digits from the MNIST dataset~\citep{LeCun2010}. For each digit $d\in[7]$, we select a random MNIST image from the digit class $d$ as the new customer type variable $s$. The images are $28\times28=784$ dimensional.

\subsubsection{Real-World Datasets}\label{appen:real}

Following previously studied causal inference methods~\citep{Shalit2017,Wu2023,Schwab2019,Bica2020}, we consider two semi-synthetic real-world datasets IHDP\footnote{IHDP: \url{https://www.fredjo.com/}.}~\citep{Hill2011} and PM-CMR\footnote{PM-CMR:\url{https://doi.org/10.23719/1506014}.}~\citep{Wyatt2020} for experiments, since the true counterfactual prediction function is rarely available for real-world datasets. 

IHDP, the Infant
Health and Development Program (IHDP), comprises 747
units with 6 pre-treatment continuous variables, one action variable and 19 discrete variables related to the children and their mothers, aiming at evaluating the effect of specialist home visits on the future cognitive test scores of premature infants. From the original data, we select all 6 continuous covariance variables as our context variable $X$.

PM-CMR studies the
impact of PM2.5 particle level on the cardiovascular mortality rate (CMR) in 2132 counties in the United States using data provided by the National Studies on Air Pollution and Health~\citep{Wyatt2020}. We use 6 continuous variables
about CMR in each city as our context variable $X$.

Following~\citet{Wu2023}, from the context variables $X$ obtained from real-world datasets, we generate the instrument $Z$, the action $A$ and the outcome $Y$ using the following model:
\begin{align*}
&Z\sim\probP(Z=z)=1/K,\quad z\in[1..K];\\
&A=\sum_{z=1}^K 1_{Z=z} \sum_{i=1}^{d_X}w_{iz}(X_i+0.2\epsilon+f_z(z))+\delta_A, \quad w_{iz}\sim\text{Unif}(-1,1);\\
&Y=9A^2-1.5A+\sum_{i=1}^{d_X} \frac{X_i}{d_X}+\abs{X_1 X_2}-\sin{(10+X_2 X_3)}+2\epsilon+\delta_Y,
\end{align*}
where $X_i$ denotes the $i$-th variable in $X$, $f_z$ is a function that returns different constants depending on the input $z$, $\delta_Y,\delta_A\sim\mathcal{N}(0,1)$ and $\epsilon\sim\mathcal{N}(0,0.1)$ are the unobserved confounders. The fully generated semi-synthetic datasets IHDP and PM-CMR have 747 and 2132 samples, respectively, and we randomly split them into training (63\%), validation (27\%), and
testing (10\%) following~\citet{Wu2023}.

\subsection{Proximal Causal Learning}\label{dataset:pcl}
Next, we provide details for benchmarking datasets of proximal causal learning. Recall that we denote $A$ as the treatment, $Y$ as the outcome, $W$ as the outcome proxy, and $V$ as the treatment proxy. The CMR problem we are solving for PCL is $\expectE[Y-f(A,W)\lvert A,V]$.

\subsubsection{Ticket Demand Dataset}\label{appen:pcl_demand}

The ticket demand dataset~\citep{Hartford2017DeepPrediction} is also extended to the PCL setting, which is first introduced by~\citet{Xu2021}. The data generating process is described by the following model:
\begin{align}
U&\sim \text{Unif}(0,10)\\
[V_1,V_2]&=[2\sin(2\pi U/10)+\epsilon_1, 2\cos(2\pi U/10)+\epsilon_2]\\
W&= 7g(U)+45+\epsilon_3\\
A&=35+(V_1+3)g(U)+V_2+\epsilon_4\\
Y&=A\cdot \min(\exp(\frac{W-A}{10},5)-5g(U)+\epsilon_5\\
\text{with}\quad g(u)&=2(\frac{(u-5)^4}{600}+\exp(-4(u-5)^2)+u/10-2)\\
\text{and}\quad \epsilon_i&\sim \mathcal{N}(0,1),
\end{align}
where $U$ is the demand, which acts as the hidden confounder, $V_1, V_2$ are fuel prices which act as treatment proxy, $W$ is the web page views which act as outcome proxy, $A$ is the price and $Y$ is the sale. Here, we can see that the outcome proxy $W$ and the treatment proxy $V$ are both affected by $U$, where $W$ directly affects the outcome and $V$ directly affects the treatment $A$.

\subsubsection{dSprites high-dimensional Dataset}\label{appen:pcl_highdim}

The dSprites dataset~\citep{dsprites17} is a high-dimensional $(64\times64)$ image dataset described by five latent parameters: \textit{shape}, \textit{scale}, \textit{rotation}, \textit{posX} and \textit{poxY}. It is proposed by~\citet{Xu2021} to adopt it as a benchmark for PCL where the treatment is each figure and the hidden confounder is \textit{posY}. For the experiments, we fix the \textit{shape} to be heart.

The data-generating process can be described by the following steps:
\begin{enumerate}
    \item Randomly generate values for \textit{scale, rotation, posX and posY}: \textit{scale} $\sim\text{Unif}\{0.5,0.6,...,1.0\}$, \textit{rotation} $\sim\text{Unif}(0,2\pi)$, \textit{posX, posY} $\sim\text{Unif}\{0,...,31\}$.
    \item Set $U$=\textit{posY}
    \item Set $V$=\textit{(scale,rotation,posX)}
    \item Set $A$ as the dSprites image with features \textit{(scale, rotation, posX, posY)} and add Gaussian noise $\mathcal{N}(0,0.1)$ to each pixel.
    \item Set $W$ as posY with Gaussian noise $\mathcal{N}(0,1)$.
    \item $Y=\frac{0.1\norm{vec(A)^TB}^2_2-5000}{1000}\times\frac{(31\times U-15.5)^2}{85.25}+\epsilon,\epsilon\sim\mathcal{N}(0,0.5)$, where the matrix $B\in\realNumber^{64\times64}$ is given by $B_{i,j}=\abs{32-j}/32$.
\end{enumerate}

\begin{figure}
    \centering
\includegraphics[width=0.5\linewidth]{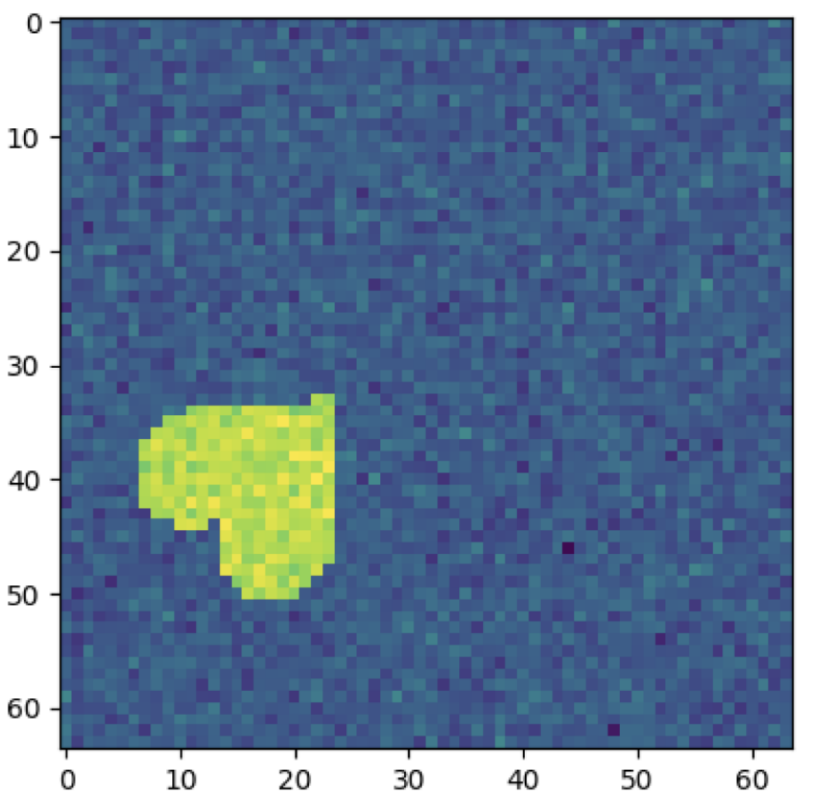}
    \caption{An example of dSprites image, which is used as the treatment $A$ in PCL experiments. Its scale, rotation, $x$ position and $y$ position are randomly generated.}
    \label{fig:dsprite_example}
\end{figure}

For the test dataset, a fixed grid of image parameters is chosen:
\begin{align}
\textit{posX}&\in[0,5,10,15,20,25,30]\\
\textit{posY}&\in[0,5,10,15,20,25,30]\\
\textit{scale}&\in[0.5,0.8,1.0]\\
\textit{rotation}&\in[0,0.5\pi,\pi,1.5\pi],
\end{align}
which consists of 588 images to reliably evaluate different PCL algorithms.

\section{Network Architecture and Hyperparameters}\label{appen:networks}

Here, we describe the network architecture and hyperparameters of all experiments. Unless otherwise specified, all neural network algorithms are optimised using AdamW~\citep{Loshchilov2017} with learning rate $= 0.001$, $\beta =(0.9,0.999)$, and $\epsilon=10^{-8}$. In addition, we set $K=10$ for $K$-fold cross-fitting in DML-CMR. In addition, all hyperparameter choices for methods and datasets used in this work are available in our code.

\subsection{IV Regression}
We first introduce details for the IV regression experiments.

\subsubsection{Ticket Demand Dataset}

For DML-CMR and CE-DML-CMR, we use the network architecture described in~\Cref{tab:demand_arch}. We use a learning rate of $0.0002$ with a weight decay of $0.001$ (L2 regularisation) and a dropout rate of $\frac{1000}{5000+N}$ that depends on the data size $N$. For DeepGMM, we use the same structure as the outcome network of DML-CMR with dropout $=0.1$ and the same learning rate as DML-CMR. For DFIV, we follow the original structure proposed in~\citet{Xu2020} with regularisers $\lambda 1$, $\lambda 2$ both set to 0.1 and weight decay of 0.001. For DeepIV, we use the same network architecture as action network and stage 2 network for DML-CMR, with the dropout rate in~\citet{Hartford2017DeepPrediction} and weight decay of 0.001. For KIV, we use the Gaussian kernel, where the bandwidth is determined by the median trick as originally described by~\citet{Singh2019}, and we use the random Fourier feature trick with 100 dimensions.

\begin{table}[t]
    \centering
    \subfloat[Action Network for $\widehat{g}$]{
    \begin{tabular}{||c|c||}
    \hline
    \textbf{Layer Type} & \textbf{Configuration}  \\ [0.5ex]
    \hline \hline
    Input & $C$\\      \hline
    FC + ReLU & in:3 out:128\\    \hline
    Dropout & - \\    \hline
    FC + ReLU & in:128 out:64\\    \hline
    Dropout & - \\    \hline
    FC + ReLU & in:64 out:32\\    \hline
    Dropout & - \\    \hline
    MixtureGaussian & 10\\    \hline

    \end{tabular}\label{tab:action_arch}}
    \hspace{30pt}
    \subfloat[Outcome Network for $\widehat{s}$]{
    \begin{tabular}{||c|c||}
    \hline
    \textbf{Layer Type} & \textbf{Configuration}  \\ [0.5ex]
    \hline \hline
    Input & $C$\\      \hline
    FC + ReLU & in:3 out:128\\    \hline
    Dropout & - \\    \hline
    FC + ReLU & in:128 out:64\\    \hline
    Dropout & - \\    \hline
    FC + ReLU & in:64 out:32\\    \hline
    Dropout & - \\    \hline
    FC & in:32 out:1\\    \hline
    \end{tabular}\label{tab:outcome_arch}}

    \vspace{10pt}
    \subfloat[Stage 2 Network for $\widehat{h}$]{
    \begin{tabular}{||c|c||}
    \hline
    \textbf{Layer Type} & \textbf{Configuration}  \\ [0.5ex]
    \hline \hline
    Input & $C,A$\\      \hline
    FC + ReLU & in:3 out:128\\    \hline
    Dropout & - \\    \hline
    FC + ReLU & in:128 out:64\\    \hline
    Dropout & - \\    \hline
    FC + ReLU & in:64 out:32\\    \hline
    Dropout & - \\    \hline
    FC & in:32 out:1\\    \hline
    \end{tabular}\label{tab:stage2_arch}}
    \caption{Network architecture for DML-CMR and CE-DML-CMR for the ticket demand low-dimensional dataset for IV regression. For the input layer, we provide the input variables. For mixture of Gaussians output, we report the number of components. The dropout rate is given in the main text.}
    \label{tab:demand_arch}
\end{table}

\subsubsection{Ticket Demand with MNIST}

For DML-CMR and CE-DML-CMR, we use a convolutional neural network (CNN) feature extractor, which we denote as \textit{ImageFeature}, described in~\Cref{tab:mnist_arch}, for all networks. The full network architecture is described in~\Cref{tab:mnist_demand_arch}; we use weight decay of 0.05. For DeepGMM, we use the same structure as the outcome network of DML-CMR, with a dropout rate of 0.1 and weight decay of 0.05. For DFIV, we follow the original structure proposed in~\citet{Xu2020} with regularisers $\lambda 1$, $\lambda 2$ both set to 0.1 and weight decay of 0.05. For DeepIV, we use the same network architecture as the action network and stage 2 network for DML-CMR, with the dropout rate in~\citet{Hartford2017DeepPrediction} and weight decay of 0.05. For KIV, we use the Gaussian kernel, where the bandwidth is determined by the median trick as originally described by~\citet{Singh2019}, and we use the random Fourier feature trick with 100 dimensions.

\begin{table}[t]
    \centering
    \begin{tabular}{||c|c||}
    \hline
    \textbf{Layer Type} & \textbf{Configuration}  \\ [0.5ex]
    \hline \hline
    Input & $28\times 28$\\      \hline
    Conv + ReLU & $3\times3\times32$, s:1, p:0\\    \hline
    Max Pooling & $2\times2$, s:2\\    \hline
        Dropout & -\\    \hline
    Conv + ReLU & $3\times3\times64$, s:1, p:0\\    \hline
    Max Pooling & $2\times2$, s:2\\    \hline
        Dropout & -\\    \hline
    Conv + ReLU & $3\times3\times64$, s:1, p:0\\    \hline
    Dropout & -\\    \hline
    FC + ReLU & in: 576, out:64\\    \hline
    \end{tabular}
    \caption{Network architecture of the feature extractor used for the ticket demand dataset with MNIST for IV regression. For each convolution
layer, we list the kernel size, input dimension and output dimension, where s stands for stride and p stands for padding. For max-pooling, we provide the
size of the kernel. The dropout rate here is set to 0.3. We denote this feature extractor as \textit{ImageFeature}.}
\label{tab:mnist_arch}
\end{table}

\begin{table}[t]
    \centering
    \subfloat[Action Network for $\widehat{g}$]{
    \begin{tabular}{||c|c||}
    \hline
    \textbf{Layer Type} & \textbf{Configuration}  \\ [0.5ex]
    \hline \hline
    Input & ImageFeature$(C),Z$\\      \hline
    FC + ReLU & in:66 out:32\\    \hline
    Dropout & - \\    \hline
    MixtureGaussian & 10\\    \hline
    \end{tabular}}
    \hspace{30pt}
    \subfloat[Outcome Network for $\widehat{s}$]{
    \begin{tabular}{||c|c||}
    \hline
    \textbf{Layer Type} & \textbf{Configuration}  \\ [0.5ex]
    \hline \hline
    Input & ImageFeature$(C),Z$\\      \hline
    FC + ReLU & in:66 out:32\\    \hline
    Dropout & - \\    \hline
    FC & in:32 out:1\\    \hline
    \end{tabular}}
    \vspace{10pt}
    \subfloat[Stage 2 Network for $\widehat{h}$]{
    \begin{tabular}{||c|c||}
    \hline
    \textbf{Layer Type} & \textbf{Configuration}  \\ [0.5ex]
    \hline \hline
    Input & ImageFeature$(C),A$\\      \hline
    FC + ReLU & in:66 out:32\\    \hline
    Dropout & - \\    \hline
    FC & in:32 out:1\\    \hline
    \end{tabular}}
        \caption{Network architecture for DML-CMR and CE-DML-CMR for the ticket demand dataset with MNIST for IV regression. For the input layer, we provide the input variables. For a mixture of Gaussians output, we report the number of components. The dropout rate is given in the main text.}
    \label{tab:mnist_demand_arch}
\end{table}

\subsubsection{IHDP and PM-CMR}
For the two real-world datasets, we use the same network architecture described in~\Cref{tab:demand_arch} as in the low-dimensional ticket demand setting, where the input dimension is increased to 7 for all networks. We use a dropout rate of 0.1 and weight decay of 0.001. For DeepGMM, we use the same structure as the outcome network of DML-CMR with dropout $=0.1$. For DFIV, we also use the same network architecture as in the low-dimensional ticket demand setting, with regularisers $\lambda 1$, $\lambda 2$ both set to 0.1 and weight decay of 0.001. For DeepIV, we use the same network architecture as the action network and stage 2 network of DML-CMR, with a dropout rate of 0.1 and weight decay of 0.001. For KIV, we use the Gaussian kernel where the bandwidth is determined by the median trick as originally described by~\citet{Singh2019}, and we use the random Fourier feature trick with 100 dimensions.

\subsection{Proximal Causal Learning}

Next, we introduce details for the proximal causal learning experiments.

\subsubsection{Ticket Demand Dataset}

For DML-CMR and CE-DML-CMR, we use the network architecture described in~\Cref{tab:PCL_demand_arch}. We use a learning rate of $0.0001$ with a weight decay of $0.001$ (L2 regularisation) for the $f$ network and $0.0001$ for the $s$ and $g$ network. The dropout rate is $\frac{400}{4000+N}$, which depends on the data size $N$. For the comparison methods, we use the default parameter values proposed in their original papers. For CEVAE, we use 1000 epochs, 0.0001 weight decay, 10 learning samples, 20 hidden dimensions, and 5 early stopping. For DFPV, $\lambda 1$, $\lambda 2$ both set to 0.1, weight decay is 0.01 for all networks, stage 1 iteration is 20, and stage 2 iteration is 1. For KPV, we set $\lambda 1$ and $\lambda 2$ to be 0.001 with data split ratio 0.5. For NMMR, learning rate is 0.003, L2 penalty is 0.000003 with network depth 4 and width 80 for the U statistics estimator. For the V statistics estimator, network depth is 3 and width is 80 while other hyperparameters remain the same. For PKDR, the number of components is 50, gamma is 50, alpha is 35, and cross validation is 5. For PMMR, $\lambda 1$ and $\lambda 2$ are 0.01, with scale 0.5.

\begin{table}[t]
    \centering
    \subfloat[Action Network for $\widehat{g}$]{
    \begin{tabular}{||c|c||}
    \hline
    \textbf{Layer Type} & \textbf{Configuration}  \\ [0.5ex]
    \hline \hline
    Input & $C$\\      \hline
    FC + ReLU & in:3 out:128\\    \hline
    Dropout & - \\    \hline
    FC + ReLU & in:128 out:64\\    \hline
    Dropout & - \\    \hline
    FC + ReLU & in:64 out:32\\    \hline
    Dropout & - \\    \hline
    MixtureGaussian & 10\\    \hline

    \end{tabular}\label{tab:PCL_action_arch}}
    \hspace{30pt}
    \subfloat[Outcome Network for $\widehat{s}$]{
    \begin{tabular}{||c|c||}
    \hline
    \textbf{Layer Type} & \textbf{Configuration}  \\ [0.5ex]
    \hline \hline
    Input & $C$\\      \hline
    FC + ReLU & in:3 out:128\\    \hline
    Dropout & - \\    \hline
    FC + ReLU & in:128 out:64\\    \hline
    Dropout & - \\    \hline
    FC + ReLU & in:64 out:32\\    \hline
    Dropout & - \\    \hline
    FC & in:32 out:1\\    \hline
    \end{tabular}\label{tab:PCL_outcome_arch}}

    \vspace{10pt}
    \subfloat[Stage 2 Network for $\widehat{h}$]{
    \begin{tabular}{||c|c||}
    \hline
    \textbf{Layer Type} & \textbf{Configuration}  \\ [0.5ex]
    \hline \hline
    Input & $C,A$\\      \hline
    FC + ReLU & in:3 out:128\\    \hline
    Dropout & - \\    \hline
    FC + ReLU & in:128 out:64\\    \hline
    Dropout & - \\    \hline
    FC + ReLU & in:64 out:32\\    \hline
    Dropout & - \\    \hline
    FC & in:32 out:1\\    \hline
    \end{tabular}\label{tab:PCL_stage2_arch}}
    \caption{Network architecture for DML-CMR and CE-DML-CMR for the ticket demand dataset for PCL. For the input layer, we provide the input variables. For mixture of Gaussians output, we report the number of components. The dropout rate is given in the main text.}
    \label{tab:PCL_demand_arch}
\end{table}

\subsubsection{dSprites dataset}

For the dSprites dataset, we adopt a CNN feature extractor to handle the image inputs for DML-CMR. The architecture of this feature extractor is provided in~\Cref{tab:PCL_mnist_arch}. We use 10 components for the mixture of Gaussian model, dropout is 0.2, batch size is 100, weight decay is 0.05, learning rate is 0.001 with Adam, and the number of epochs $int(1000000./N)+100$ depends on the sample size $N$. For CEVAE, DFPV, KPV, NMMR and PMMR, we follow the hyperparameters and network architecture used in~\citet{Kompa2022} to generate the experimental results. For PKDR, we follow the high-dimensional dataset hyperparameters used in the original paper~\citep{Wu2024} with weight decay 0.0001, 4 layers, learning rate 0.0001, and 500 epochs.

\begin{table}[t]
    \centering
    \begin{tabular}{||c|c||}
    \hline
    \textbf{Layer Type} & \textbf{Configuration}  \\ [0.5ex]
    \hline \hline
    Input & $64\times 64$\\      \hline
    Conv + ReLU & $5\times5\times64$, s:1, p:0\\    \hline
    Max Pooling & $2\times2$, s:2\\    \hline
        Dropout & -\\    \hline
    Conv + ReLU & $5\times5\times128$, s:1, p:0\\    \hline
    Max Pooling & $2\times2$, s:2\\    \hline
        Dropout & -\\    \hline
    Conv + ReLU & $5\times5\times128$, s:1, p:0\\    \hline
    Dropout & -\\    \hline
    Max Pooling & $2\times2$, s:2\\    \hline
    FC + ReLU & in: 2048, out:128\\    \hline
    \end{tabular}
\caption{Network architecture of the feature extractor used for the dSprites image dataset for PCL. For each convolution
layer, we list the kernel size, input dimension and output dimension, where s stands for stride and p stands for padding. For max-pooling, we provide the
size of the kernel. The dropout rate here is set to 0.2. We denote this feature extractor as \textit{ImageFeature}.}
\label{tab:PCL_mnist_arch}
\end{table}

\subsection{Validation and Hyper-Parameter Tuning}\label{appen:tune}

Validation procedures are crucial for tuning DNN hyperparameters and optimizer parameters. All the DML-CMR and CE-DML-CMR training stages can be validated by simply evaluating the respective losses on held-out data, as discussed in~\citet{Hartford2017DeepPrediction}. This allows independent validation and hyperparameter tuning of the two first-stage networks (the action and the outcome networks), and performs second-stage validation using the best network selected in the first stage. This validation procedure guards against the ‘weak instruments’ bias~\citep{Bound1995} that can occur when the instruments are only weakly correlated with the actions variable (see detailed discussion in~\citet{Hartford2017DeepPrediction}).

\section{Additional Experimental Results}

In this section, we provide additional experimental results including the effects of high ill-posedness (e.g., weak IVs), performance with tree-based estimators, and a hyperparameter sensitivity analysis.

\subsection{Effects of Weak Instruments}\label{appen:weak_iv}

When the correlation between instruments and the endogenous variable (the action in our case) is weak, IV regression methods generally become unreliable~\citep{Andrews2019} because the weak correlation induces variance and bias in the first stage estimator, thus inducing bias in the second stage estimator, especially for non-linear IV regressions. In theory, DML-CMR should be more resistant to biases in the first stage thanks to the DML framework, as long as the causal effect is identifiable under the weak instrument. This identifiability condition is captured in Condition~\ref{condition:dml} for DML, and is connected to the ill-posedness for CMR problems in general as discussed in~\Cref{sec:ill-posedness}. With identifiability, Theorem~\ref{thm:dml} and Corollary~\ref{coro:function_convergence} all hold, and the convergence rate guarantees still apply. Intuitively, as the ill-posedness increases, worse empirical performance will be observed.

Experimentally, for the ticket demand dataset, we alter the instrument strength by changing how much the instrument z affects the price p. Recall from~\Cref{appen:demand} that $p=25+(z+3)\psi(t)+\omega$, where $\psi$ is a nonlinear function and $\omega$ is the noise. We add an IV strength parameter $\varrho$ such that $p=25+(\varrho\cdot z+3)\psi(t)+\omega$. In~\Cref{tab:weak_iv}, we present the mean and standard deviation of the MSE of $\widehat{h}$ for various IV strengths $\varrho$ from 0.01 to 1 and sample size $N=5000$. It is very interesting to see that DML-CMR indeed performs significantly better than SOTA nonlinear IV regression methods under weak instruments.

\begin{table}[t]\setlength\extrarowheight{4pt}
\centering
\tiny
\begin{tabular}{||c|c|c|c|c|c|c||}
\hline
\textbf{IV Strength} & \textbf{1.0} & \textbf{0.8}                     & \textbf{0.6}                     & \textbf{0.4}                     & \textbf{0.2}                     & \textbf{0.01}                    \\\hline \hline
DML-CMR              & \textbf{0.0676(0.0116)} & \textbf{0.0984(0.0161)} & \textbf{0.1295(0.0168)} & \textbf{0.1859(0.0376)} & \textbf{0.2899(0.0494)} & \textbf{0.4872(0.1295)} \\\hline
CE-DML-CMR           & \textbf{0.0765(0.0119)} & \textbf{0.1064(0.0120)} & \textbf{0.1514(0.0203)} & \textbf{0.2070(0.0329)} & \textbf{0.3194(0.0572)} & \textbf{0.5302(0.1625)} \\\hline
DeepIV              & 0.1213(0.0209)          & 0.2039(0.0269)         & 0.3051(0.0415)          & 0.4476(0.0656)          & 0.6891(0.1210)          & 0.9293(0.2382)          \\\hline
DFIV                & 0.1124(0.0481)          & 0.1586(0.0320)          & 0.3080(0.1907)          & 0.8117(0.2779)          & 0.9622(0.3892)          & 1.6503(0.6845)          \\\hline
DeepGMM             & 0.2699(0.0522)          & 0.3330(0.1171)          & 0.4762(0.1056)          & 0.8666(0.2248)          & 1.0056(0.4334)          & 2.0218(0.6555)          \\\hline
KIV                 & 0.2312(0.0272)          & 0.3149(0.0218)          & 0.4275(0.0368)          & 0.6646(0.0538)          & 0.8099(0.0657)          & 1.226(0.1014)\\\hline
\end{tabular}
\caption{Results for the low-dimensional ticket demand dataset when the IV is weakly correlated with the action, plotted against IV strength. The results from this paper are shown in boldface.}
\label{tab:weak_iv}
\end{table}

\subsection{Performance of DML-CMR with tree-based estimators}\label{appen:tree-based}

The DML-CMR framework allows for general estimators following the Neyman orthogonal score function. While deep learning is flexible and widely used in SOTA non-linear IV regression methods, Gradient Boosting and Random Forests regression are all good candidate estimators for DML-CMR. In addition, as discussed in Lemma 3.3, the convergence rate and suboptimality guarantees in Theorem 3.4 and 3.5 both hold for these tree-based regressions.

Empirically, we replace the DNN estimators in DML-CMR, CE-DML-CMR, and DeepIV with Random Forests and Gradient Boosting regressors (using scikit-learn implementation). DeepIV is a good baseline for comparison, since it optimizes directly using a non-Neyman-orthogonal score and allows for direct replacement of all DNN estimators with tree-based estimators. We use 500 trees for both regressors, with minimum samples required at each leaf node of 100 for the nuisance parameters and 10 for $\widehat{h}$.

In~\Cref{tab:tree_based}, we present the mean and standard deviation of the MSE of $\widehat{h}$ with Random Forests and Gradient Boosting estimators on the ticket demand dataset with various dataset sample sizes. The results demonstrate the benefits of our Neyman orthogonal score function, and interestingly, the performance of Gradient Boosting is comparable to DNN estimators.

\begin{table}[ht]\setlength\extrarowheight{4pt}
\centering
\scriptsize
\begin{tabular}{||c|c|c|c|c||}
\hline
\textbf{IV Strength} & \textbf{Dataset Size} & \textbf{DNNs } & \textbf{Random Forests} & \textbf{Gradient Boosting} \\\hline \hline
DML-CMR    & 2000                  & \textbf{0.1308(0.0206)}             & 0.1689(0.0172)          & \textbf{0.1301(0.0112)}    \\\hline
CE-DML-CMR &    2000               & \textbf{0.1410(0.0246)}             & 0.1733(0.0198)          & \textbf{0.1329(0.0125)}    \\\hline
DeepIV    &   2000               & 0.2388(0.0438)                      & 0.2642(0.0261)          & 0.2052(0.0232)             \\\hline
DML-CMR    & 5000                  & \textbf{0.0676(0.0129)}             & 0.1067(0.0131)          & \textbf{0.0632(0.0107)}    \\\hline
CE-DML-CMR &   5000               & \textbf{0.0765(0.0119)}             & 0.1154(0.0138)          & \textbf{0.0699(0.0069)}    \\\hline
DeepIV    &    5000              & 0.1213(0.0209)                      & 0.1626(0.0128)          & 0.1020(0.0091)             \\\hline
DML-CMR    & 10000                 & \textbf{0.0378(0.0094)}             & 0.0657(0.0062)          & \textbf{0.0482(0.0079)}    \\\hline
CE-DML-CMR &  10000                & \textbf{0.0442(0.0070)}             & 0.0721(0.0039)          & \textbf{0.0523(0.0059)}    \\\hline
DeepIV    &  10000             & 0.0714(0.0140)                      & 0.1106(0.0080)          & 0.1017(0.0075)\\\hline
\end{tabular}
\caption{Results for the low-dimensional ticket demand dataset comparing the use of tree-based estimators with DNN estimators. The results from this paper are shown in boldface.}
\label{tab:tree_based}
\end{table}

\subsection{Sensitivity analysis for different Hyperparameters}\label{appen:sensitivity}

The tunable hyperparameters in DML-CMR are the learning rate, network width, weight decay, and dropout rate (see~\Cref{appen:networks}). As a sensitivity analysis, we provide results for the mean and standard deviation of the MSE of the DML-CMR estimator $\widehat{h}$ with different hyperparameter values for both the low-dimensional and high-dimensional datasets with sample size N=5000 in~\Cref{tab:ablation_low} and~\Cref{tab:ablation_high}. Overall, we see that DML-CMR is not very sensitive to small changes in the hyperparameters.

\begin{table}[t]\setlength\extrarowheight{4pt}
    \centering
    \scriptsize
    \begin{tabular}{||c|c|c|c|c|c||}
    \hline
     \textbf{Learning Rate} & \textbf{Weight Decay}& \textbf{Dropout}& \textbf{DNN Width} & \textbf{DML-CMR} &\textbf{CE-DML-CMR} \\\hline \hline
0.0002 & 0.001  & 0.1  & 128 & \textbf{0.0676(0.0129)} & \textbf{0.0765(0.0119)} \\\hline
0.0005 &        &      &     & 0.0752(0.0122)          & 0.0897(0.0196)          \\\hline
0.0001 &        &      &     & \textbf{0.0703(0.0195)} & \textbf{0.0794(0.0201)} \\\hline
       & 0.0005 &      &     & 0.0794(0.0185)          & 0.0823(0.0149)          \\\hline
       & 0.005  &      &     & 0.0765(0.0135)          & 0.0809(0.0159)          \\\hline
       & 0.01   &      &     & 0.0820(0.0162)          & 0.0865(0.0174)          \\\hline
       &        & 0.05 &     & \textbf{0.0715(0.0074)} & \textbf{0.0813(0.0089)} \\\hline
       &        & 0.2  &     & 0.0836(0.0100)          & 0.0919(0.0157)          \\\hline
       &        &      & 64  & 0.0830(0.0162)          & 0.0924(0.0121)          \\\hline
       &        &      & 256 & 0.0943(0.0179)          & 0.0981(0.0126)          \\\hline
       & 0.0005 & 0.2  &     & 0.0805(0.0133)          & 0.0910(0.0106)          \\\hline
       & 0.005  & 0.05 &     & \textbf{0.0672(0.0116)} & \textbf{0.0742(0.0102)} \\\hline
       & 0.01   & 0.05 &     & 0.0825(0.0152)          & 0.0914(0.0125)          \\\hline
       &        & 0.2  & 256 & 0.0810(0.0129)          & 0.0852(0.0121)          \\\hline
       &        & 0.05 & 64  & 0.0907(0.0149)          & 0.0963(0.0161)          \\\hline
       & 0.005  &      & 256 & 0.0939(0.0146)          & 0.0991(0.0093)\\\hline
    \end{tabular}
    \caption{Results for the low-dimensional ticket demand dataset for a range of hyperparameter values. The default hyperparameters in this case are: learning rate=0.0002, weight decay=0.001, dropout=0.1 and DNN width 128. The bold results are the best performing hyperparameters.}
    \label{tab:ablation_low}
\end{table}

\begin{table}[ht]\setlength\extrarowheight{4pt}
    \centering
     \scriptsize
    \begin{tabular}{||c|c|c|c|c|c||}
    \hline
\textbf{Learning Rate} & \textbf{Weight Decay} & \textbf{Dropout} & \textbf{CNN Channels} & \textbf{DML-CMR}          & \textbf{CE-DML-CMR}      \\\hline \hline
0.001                  & 0.05                  & 0.2              & 64                    & \textbf{0.3513(0.0125)}  & \textbf{0.3808(0.0150)} \\\hline
0.0005                 &                       &                  &                       & 0.4063(0.0129)           & 0.5008(0.0369)          \\\hline
0.002                  &                       &                  &                       & 0.3659(0.0219)           & 0.4133(0.0267)          \\\hline
0.005                  &                       &                  &                       & \textbf{0.3377(0.0218)}  & \textbf{0.3555(0.0202)} \\\hline
& 0.01                  &                  &                       & 0.3935(0.0176)           & 0.4461(0.0478)          \\\hline
& 0.02                  &                  &                       & \textbf{0.3595(0.03013)} & \textbf{0.3851(0.0293)} \\\hline
& 0.1                   &                  &                       & 0.4066(0.0172)           & 0.5160(0.0329)          \\\hline
&                       & 0.1              &                       & 0.4136(0.0211)           & 0.5386(0.0398)          \\\hline
&                       & 0.3              &                       & 0.3857(0.0171)           & 0.4002(0.0249)          \\\hline
&                       &                  & 128                   & 0.4176(0.01941)          & 0.5129(0.0630)          \\\hline
&                       &                  & 256                   & 0.4942(0.0226)           & 0.6180(0.0396)          \\\hline
& 0.1                   & 0.1              &                       & 0.4163(0.0214)           & 0.5952(0.0343)          \\\hline
& 0.01                  & 0.3              &                       & 0.3636(0.0186)           & 0.3995(0.0250)          \\\hline
&                       & 0.3              & 128                   & 0.4006(0.0187)           & 0.4764(0.0216)          \\\hline
&                       & 0.3              & 256                   & \textbf{0.3429(0.0215)}  & \textbf{0.3971(0.0264)} \\\hline
& 0.1                   &                  & 256                   & 0.4170(0.0283)           & 0.5335(0.0371)\\\hline
                       
\end{tabular}
\caption{Results for the high-dimensional ticket demand dataset for a range of hyperparameter values. The default hyperparameters in this case are: learning rate 0.001, weight decay=0.05, dropout=0.2 and 64 CNN channels. The bold results are the best performing hyperparameters.}
\label{tab:ablation_high}
\end{table}

\bibliography{main.bib}
\end{document}